\documentclass[a4paper,twocolumn]{article}          
\usepackage{graphicx}
\usepackage{mathptmx}      
\usepackage{latexsym}
\usepackage{booktabs}
\usepackage{times}
\usepackage{epsfig}
\usepackage{graphicx}
\usepackage{amsmath}
\usepackage{amssymb}
\usepackage{amsthm}
\usepackage{soul}
\usepackage{algorithm2e}

\newcommand{\tr}[1]{\textcolor{red}{\textbf{#1}}}
\newcommand{\tb}[1]{\textcolor{blue}{\textbf{#1}}}
\newcommand{\dt}[2]{\tiny{#1}$\to$\tiny{#2}}

\newtheorem{theorem}{Theorem}
\newtheorem{lemma}{lemma}

\usepackage{enumitem}
\usepackage{color}
\setlist[itemize]{leftmargin=*}
\setlist[description]{leftmargin=*}
\usepackage{authblk}

\date{\vspace{-5ex}}

\begin{document}

\title{Subspace Alignment For Domain Adaptation}




\author[1]{Basura Fernando}
\author[2]{Amaury Habrard}
\author[2]{Marc Sebban}
\author[1]{Tinne Tuytelaars}
\affil[1]{KU Leuven, ESAT-PSI, iMinds, Belgium}
\affil[2]{Universit\'{e} de Lyon, Universit\'{e} de St-Etienne F-42000, \\ UMR CNRS 5516, Laboratoire Hubert-Curien, France}

%
%
%
%
 \maketitle

\begin{abstract}
In this paper, we introduce a new domain adaptation (DA) algorithm where the 
source and target domains are represented by subspaces spanned by 
eigenvectors. Our method seeks a domain invariant feature space by learning a 
mapping function which aligns the source subspace with the target one. We show 
that the solution of the corresponding optimization problem can be obtained in a 
simple closed form, leading to an extremely fast algorithm. We present two 
approaches to determine the only hyper-parameter in our method corresponding to the 
size of the subspaces. In the first approach we tune the size of subspaces 
using a theoretical bound on the stability of the obtained result. In 
the second approach, we use maximum likelihood estimation to determine the subspace 
size, which is particularly useful for high dimensional data. Apart from PCA, we propose a subspace creation method that outperform partial least 
squares (PLS) and linear discriminant analysis (LDA) in domain adaptation. We 
test our method on various datasets and show that, despite its intrinsic 
simplicity, it outperforms state of the art DA methods.

\end{abstract}


\section{Introduction}
In classification, it is typically assumed that the test data 
comes from the same distribution as that of the labeled training data. However, 
many real 
world applications, especially in computer vision, challenge this assumption 
(see,  e.g., the study on dataset bias in \cite{A.Torralba2011}). In this 
context, the learner must take special care during the learning process to 
infer models that adapt well to the test data they are deployed on. For 
example, 
images collected from a DSLR camera are different from those taken with a web 
camera. A classifier that is trained on the former would likely fail to 
classify the latter correctly if applied without adaptation. Likewise, in face 
recognition the objective is to identify a person using available training 
images. However, the test images may arise from very different capturing 
conditions than the ones in the training set. In image annotation, 
the training images (such as ImageNet) could be very 
different from the images that we need to annotate (for example key frames 
extracted from an old video). These are some examples where 
training and test data are drawn from different distributions. 

\begin{figure*}[t]
 \centering
 \includegraphics[width=100mm]{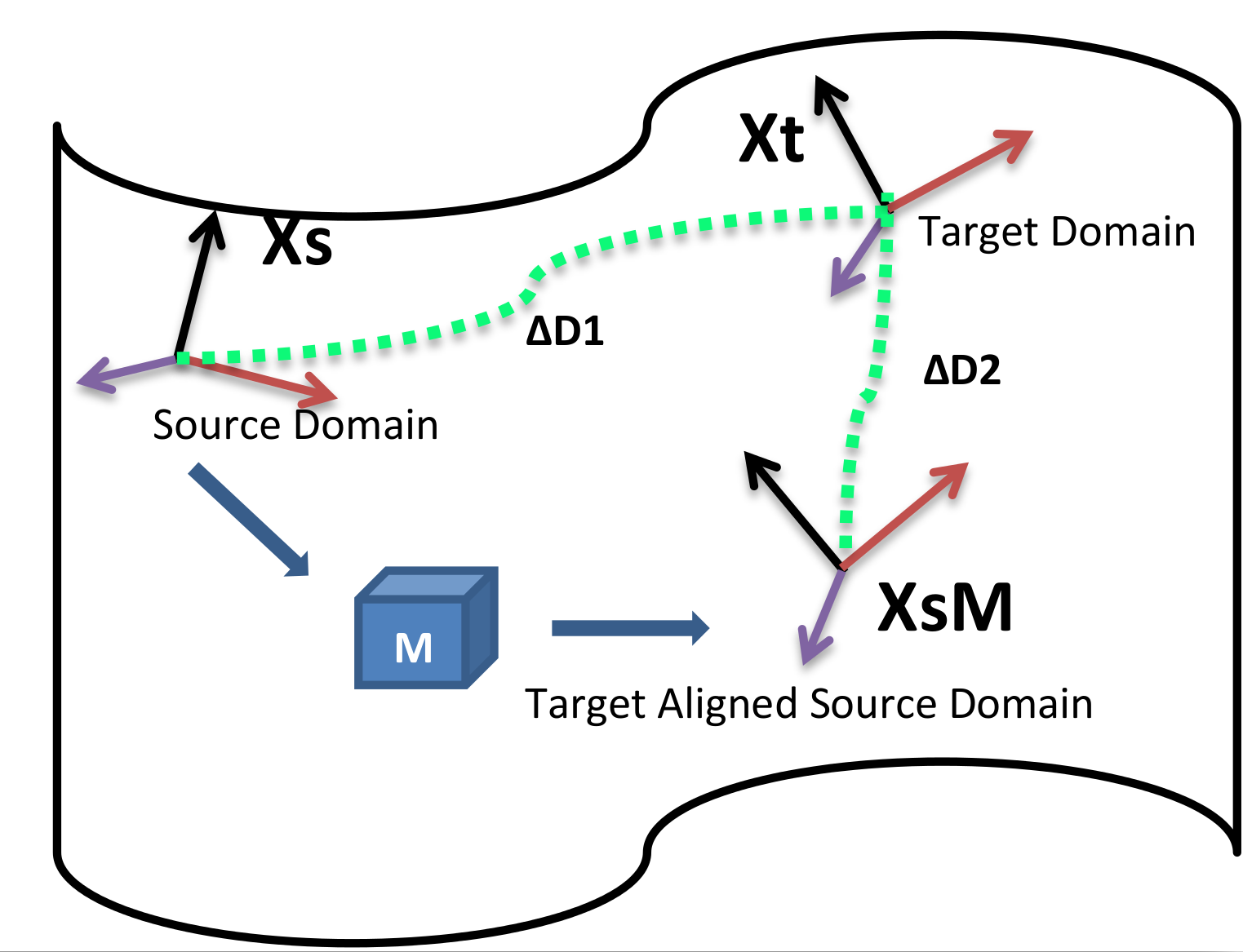} 
 \caption{Illustration of our subspace alignment method. The source domain is 
represented by the source subspace $Xs$ and the target domain by target 
subspace $Xt$. Then we align/transform the source subspace such that the 
aligned source subspace $ Xa=XsM $ is as close as possible to the target 
subspace in the Bregman divergence perspective (i.e. $ \Delta D_1 > \Delta D_2 
$ ). Then we project source data to the target aligned source subspace 
and the target data to the target subspace.}
 \label{fig:illustrated}
\end{figure*}

We refer to these different but related joint distributions as 
\textit{domains}. Formally, if we denote $P(\chi_d)$ as the data distribution 
and $P(\vartheta_d)$ the label distribution of domain $d$, then the source 
domain $S$ and the target domain $T$ have different joint distributions 
$P(\chi_S,\vartheta_S) \neq  P(\chi_T,\vartheta_T)$. In order to build robust 
classifiers, it is thus necessary to take into account the shift between these 
two distributions. Methods that are designed to overcome this shift in domains 
are known as \textit{domain adaptation} (DA) methods. DA 
typically aims at making use of information coming from both source and target 
domains during the learning process to adapt automatically. One usually 
differentiates between two scenarios: (1) the 
\textit{unsupervised} setting where the training consists of labeled source 
data and unlabeled target examples (see \cite{Margolis2011} for a survey); and 
(2) the \textit{semi-supervised} case where a large number of labels are 
available for the source domain and only a few labels are provided for the 
target domain. In this paper, we focus on the most difficult, unsupervised 
scenario.

As illustrated by recent results \cite{Gong2012,Gopalan2011}, \textit{subspace 
based domain adaptation} seems to be a promising approach to tackle 
unsupervised 
visual DA problems. In \cite{Gopalan2011}, Gopalan et al. generate intermediate 
representations in the form of subspaces along the geodesic path connecting the 
source subspace and the target subspace on the Grassmann manifold. Then, the 
source data are projected onto these subspaces and a classifier is learned. In 
\cite{Gong2012}, Gong et al. propose a geodesic flow kernel which models
incremental changes between the source and target domains. In both papers, a 
set 
of intermediate subspaces are used to model the shift between the two 
domains.

In this paper, we also make use of subspaces, one for each domain. We construct 
a subspace of size $d$, e.g. composed of the d most important eigenvectors 
induced 
by principle component analysis (PCA). However, we do not construct a set of 
intermediate subspaces. Following the theoretical recommendations of Ben-David 
et al. \cite{Ben-David2007}, we suggest to directly reduce the discrepancy 
between the two domains by moving the source and target subspaces closer. This 
is achieved by optimizing a mapping function that transforms the source 
subspace 
into the target one. Based on this simple idea, we design a new DA approach 
called \textit{subspace alignment}. The idea behind our method is illustrated 
in Figure \ref{fig:illustrated}. The source domain is represented by the source 
subspace $Xs$ and the target domain by target subspace $Xt$. Then we 
align/transform the source subspace using matrix $M$ such that the aligned 
source subspace $ Xa=M \cdot Xs $ is as close as possible to the target 
subspace in the Bregman divergence perspective. Then we project source data to 
the target aligned source subspace ($Xa$) and 
the target data to the target subspace and learn a classifier on $Xa$ subspace. 
We use this classifier to classify data in the target subspace. The 
advantage of our method is two-fold: (1) by adapting the bases of the 
subspaces, our approach is {\it global} as it manipulates the global 
co-variance matrices. This allows us to induce robust classifiers not 
subject to local perturbations (in contrast to metric learning-based 
domain adaptation approaches that need to consider pairwise or 
triplet-based constraints for inducing the metric) and (2) by 
aligning the source and target subspaces, our method is intrinsically 
regularized: we do not need to tune regularization parameters in the objective 
as imposed by a lot of optimization-based DA methods.

Our subspace alignment is achieved by optimizing a mapping function which takes 
the form of a transformation matrix $M$. We show that the optimal solution 
corresponds in fact to the covariance matrix between the source and target 
eigenvectors. From this transformation matrix, we derive a similarity function 
$Sim(\mathbf{y_S},\mathbf{y_T})$ to compare a source data $\mathbf{y_S}$ with a 
target example $\mathbf{y_T}$. Thanks to a consistency theorem, we prove that 
$Sim(\mathbf{y_S},\mathbf{y_T})$, which captures the idiosyncrasies of the 
training data, converges uniformly to its expected value. We show that we can  
make 
use of this theoretical result to tune the hyper-parameter $d$ (the 
dimensionality of the subspaces). This tends to make our method parameter-free. 
The similarity function $Sim(\mathbf{y_S},\mathbf{y_T})$ can be used directly 
in 
a nearest neighbour classifier. Alternatively, we can also learn a global 
classifier such as a support vector machine on the source data after mapping 
them onto the target aligned source subspace.

As suggested by Ben David et al.~\cite{Ben-David2007}, a reduction of the 
divergence between the two domains is required to adapt well. In other words, 
the ability of a DA algorithm to actually reduce that discrepancy is a good 
indication of its performance. A usual way to estimate the divergence consists 
in learning a linear classifier $h$ to discriminate between source and target 
instances, respectively pseudo-labeled with 0 and 1. In this context, the 
higher the error of $h$, the smaller the divergence. While such a strategy 
gives 
us some insight about the ability for a {\it global} learning algorithm (e.g. 
SVM) to be efficient on both domains, it does not seem to be suited to deal 
with {\it local} classifiers, such as the $k$-nearest neighbors. To overcome 
this limitation, we introduce a new empirical divergence specifically dedicated 
to local classifiers. We show through our experimental results that our DA 
method allows us to drastically reduce both empirical divergences.

This paper is an extended version of our previous work \cite{Fernando2013}. 
We address a few limitations of the previous subspace alignment (SA) based 
DA method. First, the work in \cite{Fernando2013} does not use source label
 information during the subspace creation. Methods such as partial least 
squares (PLS) or linear discriminant analysis (LDA) seem relevant for this
 task. However, applying different subspace creation methods such as PLS/LDA 
only for the source domain and PCA for the target domain may cause additional 
discrepancies between the source and the target domains. Moreover, LDA 
has a maximum dimensionality equal to the number of classes. To overcome 
these limitations, we use an existing metric learning algorithm 
(ITML~\cite{Davis2007}) to create subspaces in a supervised manner, 
then use PCA on both the source and the target domain. Even though 
ITML has been used before for domain adaptation (e.g. \cite{Saenko2010}), 
to the best of our knowledge, it has never been used to create linear 
subspaces in conjunction with PCA to improve subspace-based domain 
adaptation methods. We also introduce a novel large margin subspace 
alignment (LMSA) method where the objective is to learn the domain 
adaptive matrix while exploiting the source label information to 
obtain a discriminative solution. We experimentally show that both 
these methods are useful in practice.

Secondly, the cross-validation procedure presented in \cite{Fernando2013} to 
estimate the 
subspace dimensionality could be computationally expensive for high dimensional 
data as discussed in recent reviews \cite{patelvisual}. We propose a fast and 
well founded approach to estimate the subspace dimensionality and reduce 
computational complexity of our SA method. We call this new extension of SA 
method as SA-MLE. We show promising results using high dimensional data such as 
Fisher vectors using SA-MLE method. Thirdly, we present a mutual information 
based
perspective of our SA method. We show that after the adaptation, with SA method 
we can
increase the mutual information between the source and the target domains.
We also include further analysis and comparisons in the experimental section. 
We provide a detailed analysis of our DA method over several features such 
as bag-of-words, Fisher vectors~\cite{Perronnin2010} and 
DECAF~\cite{DonahueJVHZTD13} 
features. We also analyze how subspace-based DA methods perform over 
different dictionary sizes and feature normalization methods such as 
z-normalization.

The rest of the paper is organized as follows. We present the related work in 
section~\ref{sec:related}. Section~\ref{sec:asmmain} is devoted to the 
presentation of our domain adaptation method and the consistency theorem on the 
similarity measure deduced from the learned mapping function. We also present 
our supervised PCA method and the SA-MLE method 
in section~\ref{sec:asmmain}. In section~\ref{sec:experiments}, a comparative 
study is performed on various datasets. We conclude in 
section~\ref{sec:conclusion}.

\section{Related work}
\label{sec:related}

DA has been widely studied in the literature and is of great importance in many 
areas such as natural language processing \cite{Blitzer2006} and  computer 
vision \cite{A.Torralba2011}. In this paper, we focus on the unsupervised 
domain 
adaptation setting that is well suited to vision problems since it does not 
require any labeling information from the target domain. This setting makes
the problem very challenging. An important issue is to find out the 
relationship between the two domains. A classical trend is to assume the 
existence of a domain invariant feature space and  the objective of a large 
range of DA work is to approximate this space \cite{Margolis2011}. 

A classical strategy related to our work consists of learning \textit{a new
domain-invariant feature representation} by looking for a new projection space.
PCA based DA methods have then been naturally 
investigated \cite{Chen2009,Pan2008,Pan2009,Baktashmotlagh2013} in order to 
find 
a common latent space where the difference between the marginal distributions 
of 
the two 
domains is minimized with respect to the Maximum Mean Discrepancy (MMD) 
divergence. Very recently, Shao et al. \cite{Shao2014} have presented a 
very similar approach to ours. In this work Shao et al. present a low-rank 
transfer subspace learning technique that exploits the locality aware 
reconstruction in a similar way to manifold learning.   

Other strategies have been explored as well such as using metric 
learning approaches ~\cite{Kulis2011,Saenko2010} or canonical correlation 
analysis methods over different views of the data to find a coupled 
source-target subspace~\cite{Blitzer2011} where one assumes the existence of a 
performing linear classifier on the two domains. 

In the structural correspondence learning method of \cite{Blitzer2006}, Blitzer 
et al. propose to create a new feature space by identifying correspondences
among features from different domains by modeling their correlations with pivot 
features. Then, they concatenate source and target data using this feature 
representation and apply PCA to find a relevant common projection.
In~\cite{Chang2012}, Chang transforms the source data  into an intermediate 
representation such that each transformed source sample can be linearly 
reconstructed by the target samples. This is however a local approach. Local 
methods may 
fail to capture the global distribution information of the source,target 
domains. Moreover 
it is sensitive to noise and outliers of the source domain that have no 
correspondence in the target one. 
 
Our method is also related to manifold alignment whose main objective is to 
align two datasets from two different manifolds such that they can be projected 
to a common subspace \cite{Wang2009,Wang2011,Zhai2010}. Most of 
these methods~\cite{Wang2011,Zhai2010} need correspondences from the manifolds 
and all of them exploit the local statistical structure of the data unlike our 
method that captures the global distribution structure (i.e. the structure of 
the co-variances). At the same time methods such as CCA and manifold alignment 
methods can the input datasets to be from different manifolds.

Recently, subspace based DA has demonstrated good performance in visual 
DA~\cite{Gong2012,Gopalan2011}. These methods share the same principle: first 
they compute a domain specific d-dimen\-sional subspace for the  source data 
and 
another one for the target data, independently assessed by PCA. Then, they 
project source and target data into intermediate subspaces (explicitly or 
implicitly) along the shortest 
geodesic path connecting  the two d-dimensional subspaces on the Grassmann  
manifold. They actually model the distribution shift by looking for the best 
intermediate subspaces. These approaches are the closest to ours but, as 
mentioned in the introduction, it is more appropriate to align the two 
subspaces directly, instead of computing a large number of intermediate 
subspaces which 
potentially involves a costly tuning procedure. The effectiveness of our idea 
is supported by our experimental results. 

As a summary, our approach has the following differences with existing methods:

We exploit the \textit{global} co-variance statistical structure of the 
two domains during the adaptation process in contrast to the manifold 
alignment methods that use local statistical structure of the data
~\cite{Wang2009,Wang2011,Zhai2010}. We project the source data onto the 
source subspace and the target data onto the target subspace in contrast to 
methods that project source data to the target subspace or target data to the 
source subspace such as~\cite{Blitzer2011}. Moreover, we do not project data to 
a large number of subspaces explicitly or implicitly as 
in~\cite{Gong2012,Gopalan2011}. Our method 
is unsupervised and does not require any target label information like  
constraints on cross-domain data~\cite{Kulis2011,Saenko2010} or correspondences 
from across datasets~\cite{Wang2011,Zhai2010}. We do not apply PCA on 
cross-domain data like in~\cite{Chen2009,Pan2008,Pan2009} as this can help only 
if there exist shared features in both domains. In contrast, we make use of the 
correlated features in both domains. Some of these features can be specific to 
one domain yet correlated to some other shared features in both domains 
allowing us to use both shared and domain specific features.


\section{DA based on unsupervised subspace alignment}
\label{sec:asmmain}
In this section, we introduce our new subspace based DA method. We assume that 
we have a set $S$ of labeled source data (resp. a set $T$ of unlabeled 
target data) both lying in a given $D$-dimensional space and drawn i.i.d. 
according to a fixed but unknown source distribution $P(\chi_S,\vartheta_S)$ 
(resp. target distribution $P(\chi_T,\vartheta_T)$) where $P(\chi_S,\vartheta_S) 
\neq 
P(\chi_T,\vartheta_T)$. We assume that the source labeling function is more or 
less similar to the target labeling function. We denote the transpose operation 
by $'$. We denote vectors by lowercase bold fonts such as $\mathbf{y_i}$ and 
the class label by notation $L_i$. Matrices are represented by uppercase 
letters such as $X$.
 
In section~\ref{sec:notations}, we explain how to generate the source and 
target subspaces of size $d$. Then, we present our DA method  in 
section~\ref{sec:asmmethod} which consists in learning a transformation matrix 
$M$ that maps the source subspace to the target one. In section
\ref{sec:dimmain}, we present two methods to find the subspace dimensionality 
which is the only parameter in our method. We present a metric learning-based  
source subspace creation method in section \ref{sec:itmlpca} which uses labels 
of 
source data and our novel large margin subspace alignment in 
section~\ref{sec:LMSA}. In section~\ref{sec:divergence} we present a new domain 
divergence 
measure suitable for local classifiers such as the nearest neighbour 
classifier. 
Finally, in section \ref{sec:mutual} we give a mutual information based 
perspective to our method. 

\subsection{Subspace generation}
\label{sec:notations}
Even though both the source and target data lie in the same $D$-dimensional 
space, they have been drawn according to different distributions. Consequently, 
rather than working on the original data themselves, we suggest to 
handle more robust representations of the source and target domains and to 
learn the shift between these two domains. First, we transform every source 
and target data to a $D$-dimensional z-normalized vector (i.e. of 
zero mean and unit standard deviation). Note that z-normalization is an 
important step in most of the subspace-based DA methods such as GFK 
\cite{Gong2012} and GFS \cite{Gopalan2011}. Then, using PCA, we select for each 
domain the $d$ eigenvectors corresponding to the $d$ largest eigenvalues. These 
eigenvectors are used as bases of the source and target subspaces, respectively 
denoted by  $X_S$ and  $X_T$ ($ X_S , X_T \in \mathbb{R}^{D \times d}$). Note 
that $X_S'$ and  $X_T'$ are orthonormal (thus, $X_S' X_S = I_d$ and $X_T' X_T 
= I_d$ where $I_d$ is the identity matrix of size $d$). In the following, $X_S$ 
and  $X_T$ are used to learn the shift between the two domains. Sometimes, we 
refer $X_S$ and $X_T$ as subspaces, where we actually refer to the basis 
vectors of the subspace. 

\subsection{Domain adaptation with subspace alignment}
\label{sec:asmmethod}

As already presented in section~\ref{sec:related}, two main strategies are used 
in subspace based DA methods. The first one consists in projecting both source 
and target data to a common shared subspace. However, since this only exploits 
shared features in both domains, it is not always optimal. The second one aims 
to build a (potentially large) set of intermediate representations. Beyond the 
fact that such a strategy can be costly, projecting the data to intermediate 
common shared subspaces may lead to data explosion.

In our method, we suggest to project each source ($\mathbf{y_S}$) and 
target ($\mathbf{y_T}$) data (where $\mathbf{y_S}, \mathbf{y_T} \in 
\mathbb{R}^{1 \times D}$) to its respective subspace $X_S$ and $X_T$ by the 
operations $\mathbf{y_S} X_S$ and $\mathbf{y_T} X_T$, respectively. Then, we 
learn a linear transformation that maps the source subspace to the target one. 
This step allows us to directly compare source and 
target samples in their respective subspaces without unnecessary data 
projections. To achieve this task, we use a \textit{subspace alignment} 
approach. We align basis vectors by using a transformation matrix $M$ from 
$X_S$ 
to $X_T (M \in \mathbb{R} ^{d \times d})$. $M$ is learned by minimizing the 
following Bregman matrix divergence: 

\begin{equation}
F(M) = || X_S M - X_T ||_F^{2}
\label{eq:objective1}
\end{equation}

\begin{equation}
M^* = argmin_M ( F(M) )
\label{eq:objective}
\end{equation}

where $||.||_F^{2}$ is the Frobenius norm. Since $X_S$ and $X_T$ are generated 
from the first $d$ eigenvectors, it turns out that they tend to be 
intrinsically regularized\footnote{We experimented with several regularization 
methods on the transformation matrix $M$ such as 2-norm, trace norm, and 
Frobenius norm regularization. None of these regularization strategies 
improved over using no regularization.}. Therefore, we suggest not to add a 
regularization term in the Eq.~\ref{eq:objective1}. It is thus possible to 
obtain a simple solution of Eq.~\ref{eq:objective} in closed form. Because 
the Frobenius norm is invariant to orthonormal operations, we can re-write the
objective function in Eq.~\ref{eq:objective1} as follows:

\begin{eqnarray}
 M^* & = & argmin_M || X_S' X_S M - X_S' X_T ||_F^{2}\\ \nonumber
 & = &  argmin_M || M - X_S' X_T ||_F^{2}.
\label{eq:objectivesol}
\end{eqnarray}

From this result, we can conclude that the optimal $M^*$ is obtained as $M^* = 
X_S' X_T$. This implies that the new coordinate system is equivalent to $X_a = 
X_S X_S' X_T$. We call $X_a$ the \textit{target aligned source coordinate 
system}. It is worth noting that if the source and target domains are the same, 
then $X_S = X_T$ and $M^*$ is the identity matrix.

Matrix $M^*$ transforms the source subspace coordinate system into the target 
subspace coordinate system by aligning the source basis vectors with the target 
ones. 
is 

In order to compare a source data $\mathbf{y_S}$ with a target data 
$\mathbf{y_T}$, one needs a similarity function 
$Sim(\mathbf{y_S},\mathbf{y_T})$. 
Projecting $\mathbf{y_S}$ and $\mathbf{y_T}$ 
in their respective subspace $X_S$ and $X_T$ and applying the optimal 
transformation matrix $M^*$, we can define $Sim(\mathbf{y_S},\mathbf{y_T})$ as 
follows:

\begin{eqnarray}
\small
Sim(\mathbf{y_S},\mathbf{y_T}) & = & (\mathbf{y_S} X_SM^* )(\mathbf{y_T} X_T)' 
= 
\mathbf{y_S} X_S M^* X_T' \mathbf{y_T}'\nonumber\\
& =  & \mathbf{y_S} A  \mathbf{y_T}',
\label{eq:asmsym}
\end{eqnarray}

where $A= X_S X_S' X_T X_T'$. Note that Eq.~\ref{eq:asmsym} looks like a 
generalized dot product (even though $A$ is not positive semi-definite) where 
$A$ encodes the relative contributions of the different components of the 
vectors in their original space.

We use the matrix $A$ to construct the kernel matrices via 
$Sim(\mathbf{y_S},\mathbf{y_T})$ and perform SVM-based classification. 
To use with nearest neighbour classifier, we project the source data via $X_a$ 
into the 
target aligned source subspace and the target data into the target subspace 
(using $X_T$), and perform the classification in this $d$-dimensional space. 
The 
pseudo-code of this algorithm is presented in Algorithm~\ref{algo:da}.

\begin{algorithm}
\SetAlgoLined
\KwData{Source data $S$, Target data $T$, Source labels $L_S$, Subspace 
dimension $d$}
\KwResult{ Predicted target labels $L_T$ }
$X_S \leftarrow PCA(S,d)$  \;
$X_T \leftarrow PCA(T,d)$  \;
$X_a \leftarrow X_S X_S' X_T$  \;
$S_a = S  X_a$ \;
$T_T = T X_T$ \;
$L_T \leftarrow NN-Classifier(S_a,T_T,L_S)$ \;
\caption{Subspace alignment DA algorithm in the lower $d$-dimensional space}
\label{algo:da}
\end{algorithm}

\begin{figure*}[t]
 \centering
 \includegraphics[width=140mm]{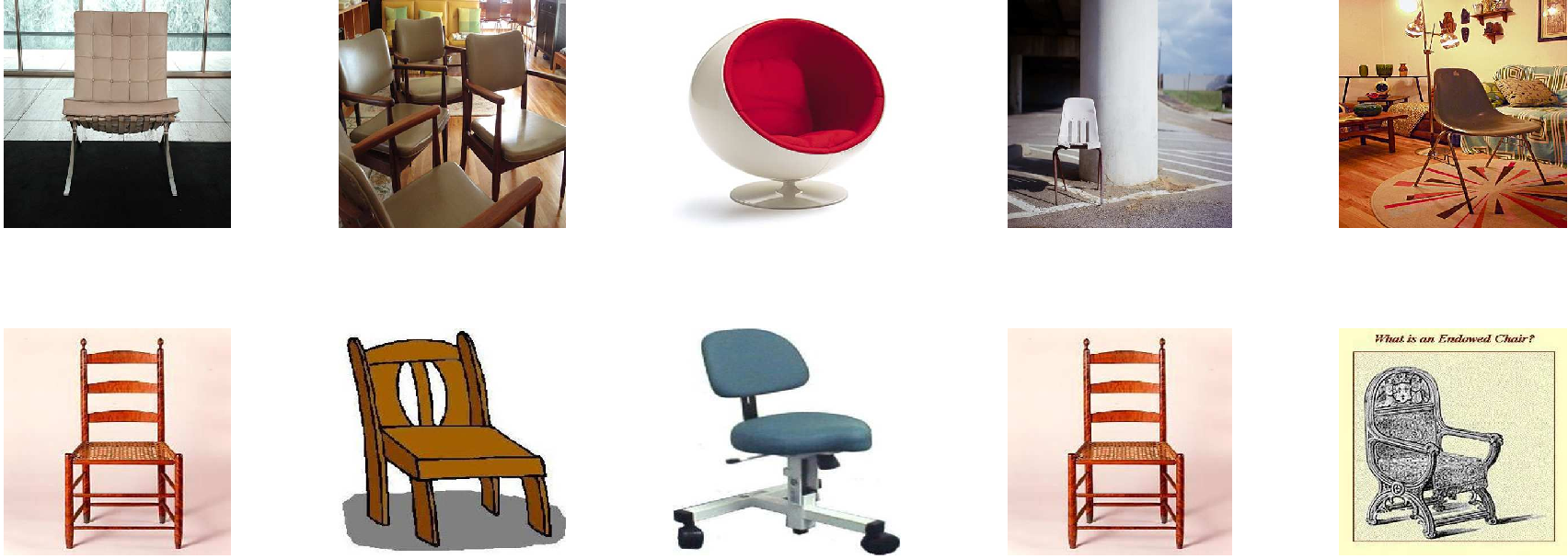} 
 \caption{Classifying ImageNet images using Caltech-256 images as the source 
domain. In the first row, we show an ImageNet query image. In the second row, 
the nearest neighbour image from our method is shown.}
 \label{fig:visual}
\end{figure*}

\begin{figure*}[t]
 \centering 
 \includegraphics[width=140mm]{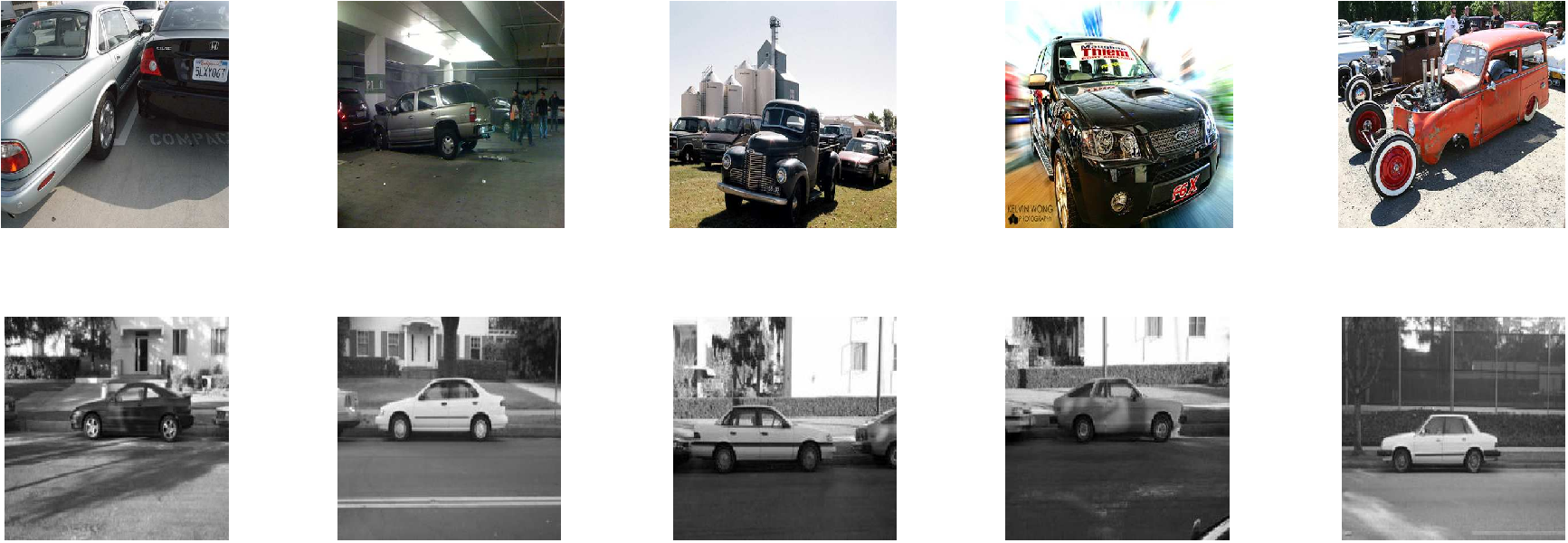} 
 \caption{Classifying ImageNet images using Caltech-256 images as the source 
domain. In the first row, we show an ImageNet query image. In the second row, 
the nearest neighbour image from our method is shown.}
 \label{fig:visual1}
\end{figure*}

From now on, we will simply use $M$ to refer to $M^*$.

\subsection{Subspace dimensionality estimation}
\label{sec:dimmain}
Now we have explained the \textit{subspace alignment} domain adaptation 
algorithm in section \ref{sec:asmmethod}, in the following two subsections we 
present two methods to find the size of the subspace dimensionality which is 
the unique hyper-parameter of our method. First, in section \ref{sec:subdim} we 
present a theoretical result on the stability of the solution and then use 
this result to tune the subspace dimensionality $d$. In the second method, we 
use maximum likelihood estimation to find the intrinsic subspace 
dimensionalities of the source domain ($d_s$) and the target domain ($d_t$). 
Then we  use these dimensionalities as before in the SA method. This slightly 
different version of SA method is called \textit{SA-MLE}. \textit{SA-MLE} does 
not necessarily need to have the same dimensionality for both the source 
subspace $X_S$ and the target subspace $X_T$. We present SA-MLE method in 
subsection \ref{sec:esa}.    

\subsubsection{SA: Subspace dimensionality estimation using consistency theorem 
on $Sim(\mathbf{y_S},\mathbf{y_T})$}
\label{sec:subdim}

The unique hyper-parameter of our algorithm is the number $d$ of eigenvectors. 
In this section, inspired from concentration inequalities on 
eigenvectors~\cite{Zwald2005}, we derive an upper bound on the similarity 
function $Sim(\mathbf{y_S},\mathbf{y_T})$. Then, we show that we can use this 
theoretical result to tune $d$.

Let $\tilde{D}_n$ be the covariance matrix of a sample $D$ of size $n$ drawn 
i.i.d. from a given distribution and $\tilde{D}$ its expected value over that 
distribution. 

\begin{theorem}\label{Zwald-blanchard05}~\cite{Zwald2005}
Let $B$ be s.t. for any $\mathbf{x}, \|\mathbf{x}\|\leq B$, let 
$X^d_{\tilde{D}}$ and 
$X^{d}_{\tilde{D}_n}$ be the orthogonal projectors of the subspaces spanned by 
the first 
d eigenvectors of $\tilde{D}$ and $\tilde{D}_n$. Let  
$\lambda_1>\lambda_2>...>\lambda_d>\lambda_{d+1}\geq 0$ be the
first $d+1$ eigenvalues of $\tilde{D}$, then for any
$n\geq \left(\frac{4B}{(\lambda_d-\lambda_{d+1})}\left(1+\sqrt{\frac{\ln( 
1/\delta)}{2}}\right)\right)^2 $ with
probability at least $1-\delta$ we have:
$$
\|X^{d}_{\tilde{D}}-X^{d}_{\tilde{D}_n}\|\leq
\frac{4B}{\sqrt{n}(\lambda_d-\lambda_{d+1})}\left(1+\sqrt{\frac{\ln( 
1/\delta)}{2}}\right).
$$
\end{theorem}

From this theorem, we can derive the following lemma for the deviation 
between  $X^{d}_{\tilde{D}} {X^{d}_{\tilde{D}}}'$ and 
$X^{d}_{\tilde{D}_n}{X^{d}_{\tilde{D}_n}}'$. For 
the sake of simplification, we will use in the following the same notation $D$ 
(resp. $D_n$) for 
defining either the sample $D$ (resp. $D_n$) or its covariance matrix 
$\tilde{D}$ (resp. $\tilde{D}_n$).

\begin{lemma}Let $B$ s.t. for any $\mathbf{x}, \|\mathbf{x}\|\leq B$, let 
$X^d_{{D}}$ and 
$X^{d}_{{D}_n}$ the orthogonal projectors of the subspaces spanned by 
the first d eigenvectors of ${D}$ and ${D}_n$. Let  
$\lambda_1>\lambda_2>...>\lambda_d>\lambda_{d+1}\geq 0$ be the
first $d+1$ eigenvalues of ${D}$, then for any
$n\geq \left(\frac{4B}{(\lambda_d-\lambda_{d+1})}\left(1+\sqrt{\frac{\ln( 
1/\delta)}{2}}\right)\right)^2 $ with
probability at least $1-\delta$ we have:
$$
\|X^{d}_{{D}}{X^{d}_{{D}}}'-X^{d}_{{D}_n}{X^{d}_{{D}_n}}'\|\leq 
\frac{8\sqrt{d}}{\sqrt{n}}\frac{B}{(\lambda_d-\lambda_{d+1})}\left(1+\sqrt{\frac
{\ln (1/\delta)}{2}}\right)
$$
\label{lema:one}
\end{lemma}
\begin{proof}
\begin{align*}
&\|X^{d}_{{D}}{X^{d}_{{D}}}'-X^{d}_{{D}_n}{X^{d}_{{D}_n}}'\|\\=&
\|X^{d}_{{D}}{X^{d}_{{D}}}'-X^{d}_{{D}}{X^{d}_{{D}_n}}'+X^{d}_{{D}}{X^{d}_{{D}_n
}}'-X^{d}_{{D}_n}{X^{d}_{{D}_n}}'\|\\
\leq&\|X^{d}_{{D}}\|\|{X^{d}_{{D}}}'-{X^{d}_{{D}_n}}'\|+\|X^{d}_{{D}}-X^{d}_{{D}
_n}\|\|{X^{d}_{{D}_n}}'\|\\
\leq&\frac{2\sqrt{d}}{\sqrt{n}}\frac{4B}{(\lambda_d-\lambda_{d+1})}\left(1+\sqrt
{\frac{\ln (1/\delta)}{2}}\right)
\end{align*}
The last inequality is obtained by the fact that the eigenvectors are 
normalized 
and thus $\|X_{{D}}\|\leq\sqrt{d}$ and application of 
Theorem~\ref{Zwald-blanchard05} twice. We now give a theorem for the projectors 
of our DA method.
\end{proof}

\begin{theorem}
Let ${X}_{{S}_{n}}^d$ (resp. ${X}_{T_{n}}^d$) be the d-dimensional projection 
operator built from the source (resp. target) sample of size $n_S$ (resp. 
$n_T$) 
and $X_{S}^d$ (resp. $X_{T}^d$) its expected value with the associated first 
$d+1$ eigenvalues $\lambda_1^S>...>\lambda_d^S>\lambda_{d+1}^S$ (resp. 
$\lambda_1^T>...>\lambda_d^T>\lambda_{d+1}^T$), then 
 we have with probability at least $1-\delta$
 $$
\|X_{{S}}^dM{X_{{T}}^d}'-X_{{S}_{n}}^dM_n{X_{{T}_n}^d}'\|\leq 
8d^{3/2}B\left(1+\sqrt{\frac{\ln( 
2/\delta)}{2}}\right) $$  $$ \times
\left(\frac{1}{\sqrt{n_S}(\lambda_d^S-\lambda_{d+1}^S)} 
+\frac{1}{\sqrt{n_T}(\lambda_d^T-\lambda_{d+1}^T)}\right)
$$
where $M_n$ is the solution of the optimization problem of 
Eq~\ref{eq:objective} using source and target samples of sizes $n_S$ 
and $n_T$ respectively, and $M$ is its expected value. 
\label{theorem:asm}
\end{theorem}
\begin{proof}
\small
\begin{align*}
&\|X_{{S}}^dM{X_{{T}}^d}'-X_{{S}_{n}}^dM_n{X_{{T}_n}^d}'\|=\\
&\|X_{{S}}^d{X_{{S}}^d}'X_{{T}}^d{X_{T
}^d}'-X_{{S}_{n}}^d{X_{{S}_{n}}^d}'X_{{T}_n}^d{X_{{T}_n}^d}'\|\\
=&\|X_{{S}}^d{X_{{S}}^d}'X_{{T}}^d{X_{{T}}^d}'-X_{{S}}^d{X_{{S}}^d}'X_{{T}_n}^d{
X_{{T}_n}^d}'+ \\ & X_
{S}^d{X_{{S}}^d}'X_{{T}_n}^d{X_{{T}_n}^d}'-X_{{S}_{n}}^d{X_{{S}_{n}}^d}'X_{{T}_n
}^d{X_{{T}_n}^d}
'\|\\
\leq&\|X_{{S}}^d{X_{{S}}^d}'X_{{T}}^d{X_{{T}}^d}'-X_{{S}}^d{X_{{S}}^d}'X_{{T}_n}
^d{X_{{T}_n}^d}
'\|+ \\ & 
\|X_{{S}}^d{X_{{S}}^d}'X_{{T}_n}^d{X_{{T}_n}^d}'-X_{{S}_{n}}^d{X_{{S}_{n}}^d
}'X_{{T}_n}^d{X_{
T_n}^d}'\|\\
\leq&\|X_{{S}}^d\|\|{X_{{S}}^d}'\|\|X_{{T}}^d{X_{{T}}^d}'-X_{{T}_n}^d{X_{{T}_n}
^d}'\|+ \\ & \|X_{S
}^d{X_{{S}}^d}'-X_{{S}_{n}}^d{X_{{S}_{n}}^d}'\|\|X_{{T}_n}^d\|\|{X_{{T}_n}^d}
'\|\\
\leq&8d^{3/2}B\left(1+\sqrt{\frac{\ln( 
2/\delta)}{2}}\right) \times \\ & 
\left(\frac{1}{\sqrt{n_S}(\lambda_d^S-\lambda_{d+1}^S)}
+\frac{1}{\sqrt{n_T}(\lambda_d^T-\lambda_{d+1}^T)}\right).
\end{align*}
\end{proof}
The first equality is obtained by replacing $M$ and $M_n$ by their 
corresponding 
optimal solutions $X_S^d{X_T^d}'$ and $X_{{S}_{n}}^d{X_{{T}_n}^d}'$ from  
Eq~\ref{eq:objectivesol}. The last inequality is obtained by applying 
Lemma~\ref{lema:one} twice and bounding the projection operators.

From Theorem~\ref{theorem:asm}, we can deduce a bound on the deviation between 
two successive 
eigenvalues. We can make use of this bound as a cutting rule for automatically 
determining the size of the subspaces.
Let $n_{min}=\min(n_S,n_T)$ and 
$(\lambda_d^{min}-\lambda_{d+1}^{min})=\min((\lambda_d^T-\lambda_{d+1}^T),
(\lambda_d^S-\lambda_{d+1}^S))$ and let $\gamma>0$ be a given allowed deviation 
such that:
$$
\gamma\geq \left(1+\sqrt{\frac{\ln 
2/\delta}{2}}\right)\left(\frac{16d^{3/2}B}{\sqrt{n_{min}}(\lambda_d^{min}
-\lambda_{d+1}^{min})}\right).
$$ Given a confidence $\delta>0$ and a fixed deviation $\gamma>0$, we can 
select the maximum dimension $d_{max}$ such that:
\begin{equation}
(\lambda_{d_{max}}^{min} -\lambda_{d_{max}+1}^{min})\geq 
\left(1+\sqrt{\frac{\ln 
2/\delta}{2}}\right)\left(\frac{16d^{3/2}B}{\gamma\sqrt{n_{min}}}\right).
\label{eq:bound}
\end{equation}

For each $ d \in \{d |1 \ldots d_{max} \}$, we then have the guarantee that 
$\|X_{{S}}^dM{X_{{T}}^d}'-X_{{S}_{n}}^dM_n{X_{{T}_n}^d}'\| \leq \gamma$. In 
other words, as long as we select a subspace dimension d such that $ d \leq 
d_{max} $, the solution $M^*$ is stable and not prone to over-fitting.

Now we use this theoretical result to obtain the subspace dimensionality 
for our method. We have proved a bound on the deviation between two successive 
eigenvalues. We use it to automatically determine the maximum size 
of the subspaces $d_{max}$ that allows to get a stable and non over-fitting 
matrix $M$. To find a suitable subspace dimensionality $(d^* < d_{max})$, we 
consider all the subspaces of size $d = 1$ to $d_{max}$ and select the best 
$d$ that minimizes the classification error using a two fold cross-validation 
over the labeled source data. Consequently, we assume that the source and the 
target subspaces have the same dimensionality. For more details on the 
cross-validation procedure see section \ref{sec:subdim2}.

\subsubsection{SA-MLE: Finding the subspace dimensionality using maximum 
likelihood estimation }
\label{sec:esa}

In this section we present an efficient subspace dimensionality estimation 
method using the maximum likelihood  estimation to be used in \textit{Subspace 
Alignment} for high dimensional data when NN classifier is used. We call 
this extension of our SA method \textit{SA-MLE}. There are two problems 
when applying SA method on high dimensional data (for example, Fisher vectors 
\cite{Perronnin2010}). First, the subspace dimensionality estimation using a  
cross-validation procedure as explained in section \ref{sec:subdim} and 
section \ref{sec:subdim2} could be computationally expensive. Secondly, since 
we compute the similarity metric in Eq. (\ref{eq:asmsym}) in the original 
$R^D$ space, the size of the resulting similarity metric $A$ is $R^{D \times 
D}$. For high dimensional data, this is not efficient. 

However, after obtaining the domain transformation matrix $M$ (using
Eq. \ref{eq:objective1}), any classification learning problem can be formulated 
in the 
\textit{target aligned source subspace} (i.e. $X_S M$) which has smaller 
dimensionality than the original space $R^D$. To reduce the computational 
effort, we propose to evaluate the sample dissimilarity 
($D_{SA-MLE}(\mathbf{y_{s}},\mathbf{y_{s}})$) between the target aligned source 
samples and the target subspace projected target samples using Euclidean 
distance as 
in Eq. \ref{eq:distance_eqESA} where $\mathbf{y_s}$ is the source sample and 
$\mathbf{y_t}$ is the target sample.

\begin{equation}
D_{SA-MLE}(\mathbf{y_{s}},\mathbf{y_{t}})= 
||\mathbf{y_{s}}X_SM-\mathbf{y_{t}}X_T||_2.
\label{eq:distance_eqESA}
\end{equation}


The source and target subspaces $X_s$ and $X_t$ are obtained by 
PCA as before. But now the size of the subspaces are different: 
the source subspace is of size $d_s$ and the target subspace is of size $d_t$.

dimensionality 
source 


One objective of the method described in this section is to retain the 
local neighborhood information after dimensionality reduction. The key reason 
for selecting this approach is that on average it preserve the local 
neighborhood information after the dimensionality reduction. This allows to 
preserve useful information in respective domains while adapting the source 
information to the target domain. With this purpose, we choose the domain 
intrinsic dimensionality obtained through the method presented in 
\cite{Levina2004}. Its objective is to derive the maximum likelihood estimator 
(MLE) of the dimension $d$ from i.i.d. observations. 

MLE estimator assumes that the observations represent an embedding of a lower 
dimensional sample. For instance, we can write $\mathbf{y} = \phi(\mathbf{z})$ 
where $\phi$
is a continuous and sufficiently smooth mapping, and $\mathbf{z}$  are sampled 
from a smooth density function $f$ on $\mathbf{R}^d$, with unknown 
dimensionality $d$ with $d<D$. In this setting, close neighbors in 
$\mathbf{R}^d$ are mapped to close neighbors in the embedding $\mathbf{R}^D$. 
Let's fix a point $\mathbf{y}$ and assume $f(\mathbf{y}) \approx const$ in a 
small sphere $S_{\mathbf{y}}(R)$ of radius $R$ around $\mathbf{y}$. The 
binomial process $\{N(t, \mathbf{y}); 0 \le t \le R\}$ which counts the 
observations within distance $t$ from $\mathbf{y}$ is

\begin{eqnarray}
N(t,\mathbf{y}) = \sum_{i=1}^{n} \textbf{1} \{ \mathbf{y}_i \in 
S_{\mathbf{y}}(t)\}~,
\label{eq:process}
\end{eqnarray}
where $\textbf{1}\{ \cdot \}$ is the indicator function.
By approximating (\ref{eq:process}) with a Poisson process it can be shown 
that the maximum likelihood estimate of the intrinsic dimensionality 
for the data point $\mathbf{y}$ is:
\begin{eqnarray}
\widehat{d}(\mathbf{y}) = \left[ \frac{1}{N(R,\mathbf{y})} 
\sum_{j=1}^{N(R,\mathbf{y})} \log \frac{R}{\Theta_j(\mathbf{y})}  \right]^{-1}
\label{eq:estimate}
\end{eqnarray}
where $\Theta_j(\mathbf{y})$ is the distance from sample $\mathbf{y}$ to its 
$j^{th}$ nearest neighbor \cite{Levina2004}.

For our experiments we set $R$ to the mean pair-wise 
distance among the samples. The intrinsic dimensionality of a domain 
is then obtained by the average 
$\widehat{d}=\frac{1}{n}\sum_{i=1}^n\widehat{d}(\mathbf{y}_i)$ 
over all its instances. The two domains are considered separately, which 
implies 
$d_S\neq d_T$~. For MLE based dimensionality estimation we use 
the implementation of \cite{VanderMaaten2008}.

\subsection{Using the labels of the source domain to improve subspace 
representation}
\label{sec:itmlpca}

In this paper we focus on unsupervised domain adaptation. So far we have 
created subspaces using PCA, which does not use available label information 
even 
from the source domain. In this section we further investigate whether 
the available label information from the source domain can be exploited to 
obtain a better source subspace. One possibility is to use a supervised 
subspace 
creation (dimensionality reduction) method such as partial least squares (PLS) 
or linear discriminant analysis (LDA). For example \cite{Gong2012} uses PLS to 
exploit the label information in the source domain. 

Using PLS or LDA for creating subspaces has two major issues. First, 
the subspaces created by LDA have a dimensionality equal to the number of 
classes. This clearly is a limitation. Moreover, using PLS or LDA only for the 
source domain and PCA for the target domain causes an additional discrepancy 
between the source 
and the target subspaces (as they are generated from different methods). For 
example, in the case of SA method, it is not clear how to apply consistency 
theorem on PLS-based subspaces. The subspace disagreement 
measure~\cite{Gong2012} (SDM)
 uses principle angles of subspaces to find a good subspace dimensionality.
We believe it is not valid to use PLS for the source domain and PCA for the 
target domain 
when SDM is used as the principle angles generated by PLS and PCA have different 
meanings. 
To overcome, these issues we propose a simple, yet interesting and 
effective method to create supervised subspaces for subspace-based DA methods.

Our method is motivated by recent metric learning-based cross-domain adaptation 
method of \cite{Saenko2010}. Saenko et al \cite{Saenko2010}, used information 
theoretic metric learning method \cite{Davis2007} (ITML) to construct 
cross-domain constraints to obtain a distance metric in semi-supervised domain 
adaptation. They use labelled samples from both the source and the target 
domains to construct a distance metric that can be used to compare a source 
sample with target domain samples. We also use the information theoretic metric 
learning (ITML) method \cite{Davis2007} to create a distance metric only for 
the 
source domain using the labeled source samples. In contrast to the work of 
Saenko et al. \cite{Saenko2010}, we use this learned metric to transform source 
data into the metric induced space such that the discriminative nature of the 
source data is preserved. Afterwards, we apply PCA on the transformed source 
data samples to obtain the eigenvectors for our subspace alignment method. The 
proposed new algorithm to create supervised source subspace is shown in Algo. 
\ref{algo:itmlxs}. 

The advantage of our method is that the source and target subspaces are still 
created using the same PCA algorithm as before. As a result we can still use 
the 
consistency theorem to find a stable subspace dimensionality. In addition, we 
show that not only it improves results for our SA method, but also for other 
subspace based methods such as GFK. 

\begin{algorithm}
\SetAlgoLined
\KwData{Source data $S$, Source labels $L_S$}
\KwResult{ Source subspace $X_s$}
1. Learn the source metric $W \leftarrow ITML(S,L_S)$  \;
2. Apply cholesky-decomposition to $W$. $W_c \leftarrow chol(W)$  \;
3. Project source data to $W_c$ induced space. $S_w \leftarrow S W_c$  \;
4. $X_s \leftarrow PCA(S_w)$ \;
\caption{Source Subspace learning algorithm with metric learning - (ITML-PCA 
method)}
\label{algo:itmlxs}
\end{algorithm}

\subsection{LMSA: Large-Margin Subspace Alignment}
\label{sec:LMSA}
Motivated by the above approach, we propose a more founded optimization 
strategy to incorporate the class information from the source domain during DA 
learning procedure. To preserve discriminative information in the source domain, 
we try to minimize pairwise distances between samples from the same class. At 
the same time utilizing triplets, we make sure that samples from different 
classes are well separated by a large margin. To achieve this motivation we 
modify the objective function in Eq. \ref{eq:objective1} as follows:

\begin{multline}
F(M) = || X_S M - X_T ||_F^{2} +  \beta_1 \sum_{(i,j) \in \Omega } d_{i,j}  + 
\\ 
\beta_2 \sum_{i,j,k} max( 0,1- (d_{i,k}-d_{i,j}) )
\label{eq:objective4}
\end{multline}

where $\Omega = \{ i,j \}$ is the set of all source image pairs from the same 
class. The triplet $\{i,j,k\}$ are created such that $\{i,j\}$ are from the same 
class and $j \in 3NN(i)$ while $\{i,k\}$ are from different classes ($i,j,k$ are 
from the source domain). We generate all possible such triplets. The parameters 
$\beta_1 , \beta_2>0$. The Euclidean distance ($d_{i,j}$) between pair of images 
projected into target aligned source subspace is defined as follows: 

\begin{equation}
d_{i,j}= ||\mathbf{y_{S_i}}X_SM-\mathbf{y_{S_j}}X_SM||_2.
\label{eq:distance_eq}
\end{equation}

We call this approach as LMSA.


\subsection{Divergence between source and target domains}
\label{sec:divergence}

If one can estimate the differences between domains, then this information can 
be used to estimate the difficulty of adapting the source domain to a specific 
target domain. Further, such domain divergence measures can be used to evaluate 
the effectiveness of domain adaptation methods. In this section we discuss two 
domain divergence measures, one suitable for global classifiers such as support 
vector machines and the other useful for local classifiers such as nearest 
neighbour classifiers. Note we consider, nearest neighbour as a local 
classifier 
as the final label of the test sample only depends on the local neighborhood 
distribution of the training data. At the same time we consider a SVM 
classifier 
as a global classifier as SVM possibly depends on all training samples.

Ben-David et al.~\cite{Ben-David2007} provide a 
generalization bound on the target error which depends on the source error and 
a measure of divergence, called the $H \Delta H$ divergence, between the source 
and target distributions $P(\chi_S)$ and $P(\chi_T)$. 

 \begin{equation}
 \epsilon_T(h) =  \epsilon_S(h) + d_{H \Delta H} (P(\chi_S),P(\chi_T)) + 
\lambda,
\label{eq:bendavid}
 \end{equation}

where $h$ is a learned hypothesis, $\epsilon_T(h)$ the generalization 
target error, $\epsilon_S(h)$ the generalization source error, and $\lambda$ 
the 
error of the ideal joint hypothesis on $S$ and $T$, which is supposed to be a 
negligible term if the adaptation is possible. 
Eq.~\ref{eq:bendavid} tells us that to adapt well, one has to learn a 
hypothesis 
which works well on $S$ while reducing the divergence between $P(\chi_S)$ and 
$P(\chi_T)$. To estimate $d_{H \Delta H} (P(\chi_S),P(\chi_T))$, a usual way 
consists in 
learning a linear classifier $h$ to discriminate between source and target 
instances, respectively pseudo-labeled with 0 and 1. In this context, the 
higher 
the error of $h$, the smaller the divergence. While such a strategy gives us 
some insight about the ability for a {\it global} learning algorithm (e.g. SVM) 
to be efficient on both domains, it does not seem to be suited to deal with 
{\it 
local} classifiers, such as $k$-nearest neighbors. To overcome this 
limitation,  we introduce a new empirical divergence specifically designed for 
local classifiers. Based on the recommendations of~\cite{Ben-David2012}, we 
propose a discrepancy measure to estimate the local density of a target point 
w.r.t. a given source point. This discrepancy, called \textit{Target density 
around source} \textbf{TDAS} counts on average how many target points can be 
found within a $\epsilon$ neighborhood of a source point. More formally: 

\begin{equation}
 TDAS = \frac{1}{n_S} \sum_{\forall \mathbf{y_S} } |\{ \mathbf{y_T} |
Sim(\mathbf{y_S},\mathbf{y_T})  \geq \epsilon \}|.
\label{eq:tdas}
\end{equation}

Note that \textbf{TDAS} is associated with similarity measure  
$Sim(\mathbf{y_S},\mathbf{y_T})=\mathbf{y_S} A \mathbf{y_T}'$ where $A$ is the 
learned metric. As we will see in the next section, \textbf{TDAS} can be used 
to 
evaluate the effectiveness of a DA method under the covariate shift assumption 
and probabilistic Lipschitzness assumption~\cite{Ben-David2012}. The larger the 
TDAS, the better the DA method.

\subsection{Mutual information perspective on subspace alignment}
\label{sec:mutual}
Finally, in this section we look at subspace alignment from a mutual 
information point of view. 
We start this discussion with a slight abuse of notations. We denote $H(S)$ as 
the entropy of the source data $S$ and $H(S,T)$ the cross entropy between 
the source data $S$ and the target data $T$. The mutual information between the 
source domain $S$ and the target domain $T$ then can be given as follows:

\begin{equation}
\begin{split}
MI(S;T) & = H(S) + H(T) - H(S,T)   \\
 & = H(T)   - D_{KL} (S || T). 
\end{split}
\label{eq:MI}
\end{equation}

According to Eq. (\ref{eq:MI}), if we maximize the target entropy (i.e. $H (T) 
$) and minimize the \textit{KL-divergence} (i.e. $D_{KL} (S || T)$ term), we 
maximize the mutual information $MI(S;T)$. In domain adaptation it makes sense 
to maximize the mutual information between the source domain and the target 
domain. If one projects data from all domains to the target subspace $X_T$, 
then 
this allows to increase the entropy term $H(T)$, hence the mutual information 
$MI(S;T)$. This could be a reason why projecting to the target subspace 
performs quite well for some DA problems as reported by our previous work 
\cite{Fernando2013} and other related work \cite{Gong2012}. We also project the 
target data to the target subspace which allows to maximize the term $H(T)$.

For the simplicity of derivation, we assume the source and the target data are 
drawn from a zero centered Gaussian distribution. Then the 
\textit{KL-divergence} between the source domain and the target domain can be 
written as follows:

\begin{eqnarray}
D_{KL} (S || T) = \frac{1}{2}  ( tr( \Sigma_t^{-1} \Sigma_s ) - d  
-ln( \frac{det(\Sigma_s)}{det(\Sigma_t)})) 
\label{eq:MIDKL}
\end{eqnarray}

where $\Sigma$ is the covariance matrix, $tr()$ is the trace of a matrix and 
$det()$ is the determinant of the matrix. The term $ \Sigma_t^{-1} 
\Sigma_s$ can be written as follows:

\begin{eqnarray}
\Sigma_t^{-1} \Sigma_s = X_s  \Lambda_s X'_s X_t \Lambda_t^{-1} X'_t
\label{eq:trace}
\end{eqnarray}

where $\Lambda$ is a diagonal matrix constituted of eigenvalues $\lambda$. 

The term in Eq. \ref{eq:trace} is minimized if basis vectors of $X_s$ are 
aligned with the basis vectors in $X_t$. In SA method, by aligning the two 
subspaces we indirectly minimize the $tr( \Sigma_t^{-1} \Sigma_s)$ term. 
Aligning all basis vectors does not contribute equally as the basis 
vectors with higher eigenvalues ($\lambda_s,\lambda_t$) are the most 
influential. 
So we have to align the most important $d$ number of basis vectors. This step 
allows us to reduce the KL-divergence between the source and the target domain. 
In summary, the subspace alignment method optimizes both criteria given in 
equation \ref{eq:MI}, which allows us to maximize the mutual information 
between 
the source distribution and the target distribution.

\section{Experiments}
\label{sec:experiments}

We evaluate our domain adaptation approach in the context of object recognition 
using a standard dataset and protocol for evaluating visual domain adaptation
methods as in~\cite{Chang2012,Gong2012,Gopalan2011,Kulis2011,Saenko2010}. 
In addition, we also evaluate our method using various other image 
classification datasets. We denote our subspace alignment method by SA and the 
maximum likelihood estimation based SA methods as SA-MLE. Unless specifically 
mentioned, by default we consider SA method in rest of the experiments.

First, in section \ref{sec:datasets} we present the datasets. Then in section 
\ref{sec:expsetup} we present the experimental setup. Experimental details 
about subspace dimensionality estimation for the consistency theorem-based SA 
method are presented in section \ref{sec:subdim2}. Then we evaluate considered 
subspace-based methods using two domain divergence measures in section 
\ref{sec:divergenceexp}. In section \ref{sec:classification-results} we 
evaluate the classification performance of the proposed DA method using three 
DA-based object recognition datasets. Then in section \ref{sec:expITML} we 
present the effectiveness of supervised subspace creation-based SA method. We 
analyze the performance of subspace-based DA methods when used with high 
dimensional data such as Fisher vectors in section \ref{sec:expESA}. In section 
\ref{sec:anaDictionary} we analyze the performance of DA methods by varying the 
dictionary size of the bag-of-words. We evaluate the effectiveness of 
subspace based DA methods using modern deep learning features in section 
\ref{sec:expDecaf}. The influence of z-normalization on DA methods is analyzed 
in section \ref{sec:zscore}. Finally, we compare our SA method with 
non-subspace-based methods in section~\ref{sec:nonsub}.

\subsection{DA datasets and data preparation}
\label{sec:datasets} 
 
We provide three series of classification experiments on different datasets. In 
the first series, we use the Office+Caltech-10~\cite{Gong2012} dataset that 
contains four domains altogether to evaluate all DA methods (see Figure 
\ref{fig:office-clatech-dataset}). The Office dataset 
\cite{Saenko2010} consists 
of images from web-cam (denoted by \textbf{W}), DSLR images (denoted by 
\textbf{D}) and Amazon images (denoted by \textbf{A}). The Caltech-10 images 
are 
denoted by \textbf{C}. We follow the same setup as in~\cite{Gong2012}. 
We use each source of images as a domain, consequently we get four domains 
(\textbf{A}, \textbf{C}, \textbf{D} and \textbf{W}) leading to 12 DA
problems. We denote a DA problem by the notation $ S \to T $. We use the image 
representations provided by~\cite{Gong2012} for Office and Caltech10 datasets 
(SURF features encoded with a visual dictionary of 800 words). We follow the 
standard protocol of~\cite{Gong2012,Gopalan2011,Kulis2011,Saenko2010} for
generating the source  and target samples.

\begin{figure*}[t]
 \centering
 \includegraphics[width=160mm]{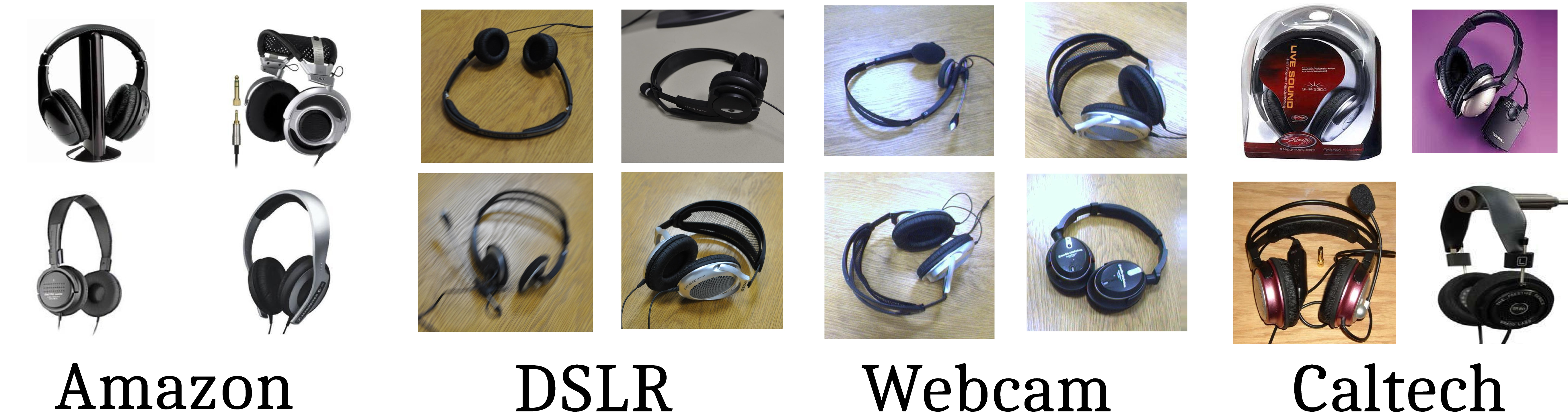}  
 \caption{Some example images from Office-Caltech dataset. This dataset consists 
of four visual domains, namely images collected from Amazon merchant website, 
images collected from a high resolution DSLR camera, images collected from a web 
camera and images collected from Caltech-101 dataset.}
 \label{fig:office-clatech-dataset}
 \end{figure*}

In a second series, we  evaluate the effectiveness of our DA method using other 
datasets, namely  ImageNet (\textbf{I}), LabelMe (\textbf{L}) and Caltech-256 
(\textbf{C}). In this setting we consider each dataset as a domain. We select 
five object categories common to all three datasets (bird, car, chair, dog and 
person) leading to a total of 7719 images. We extract dense SIFT features and 
create a  
bag-of-words dictionary of 256 words using kmeans. Afterwards, we use LLC 
encoding and a spatial pyramid ($2 \times 2$ quadrants + $3 \times 1$ 
horizontal 
+ 1 full image) to obtain a 2048 dimensional image representation (similar data 
preparation as in~\cite{Khosla2012}). 

In the last series,  we evaluate the effectiveness of our DA method using 
larger datasets, namely PASCAL-VOC-2007 and ImageNet. We select all the 
classes of PASCAL-VOC-2007. The objective here is to classify PASCAL-VOC-2007 
test images using classifiers that are learned from the ImageNet dataset. To 
prepare the data, we extract dense SIFT features and create a bag-of-words 
dictionary of 256 words using only ImageNet images. Afterwards, we use LLC 
encoding 
and spatial pyramids ($2 \times 2$ + $3 \times 1$ + 1) to obtain a 2048 
dimensional image representation. 

\subsection{Experimental setup}
\label{sec:expsetup} 

We compare our subspace  DA approach with three other DA methods and three 
baselines. Each of these methods defines a new representation space and our 
goal is to compare the performance of a 1-Nearest-Neighbor (NN) classifier and 
a linear SVM classifier on DA problems in the subspace found.

We naturally consider \textit{Geodesic Flow  Kernel} 
(\textbf{GFK} \cite{Gong2012}),  \textit{Geodesic Flow Sampling} 
(\textbf{GFS} \cite{Gopalan2011}) and \textit{Transfer Component Analysis} 
(\textbf{TCA}) \cite{Pan2009}. They have indeed demonstrated state of the art 
performances achieving better results than metric learning 
methods~\cite{Saenko2010} and better than those reported by Chang's method 
in~\cite{Chang2012}. Moreover, these methods are conceptually the closest to our 
approach. We 
also report results obtained by the following three baselines:

\begin{itemize}
 \item  \textbf{Baseline-S:} where we use the projection defined by the PCA 
subspace 
$X_S$ built from the source domain to project both source and target data and 
work in the resulting representation. 

\item \textbf{Baseline-T:} where we use 
similarly the projection defined by the PCA subspace $X_T$ built from the 
target 
domain. 
\item No adaptation \textbf{NA:} where no projection is made, i.e. we use  the 
original input space without learning a new representation.
\end{itemize}

For each method, we compare the performance of a 1-Nearest-Neighbor (NN) 
classifier and of a linear SVM classifier (we seek the best C parameter around 
the mean similarity 
value obtained from the training set) in the subspace defined by each method.
For each source-target DA problem in the first two series of experiments, we 
evaluate the accuracy of each method on the target domain over 20 random 
trials. For each trial, we consider an unsupervised DA setting where we 
randomly sample labeled data in the source domain as training data and 
unlabeled data in the target domain as testing examples. For the first series 
we use the typical setup as in \cite{Gong2012}. For the second series we use a 
maximum of 100 randomly sampled training images per-class. In the last series 
involving the PASCAL-VOC dataset, we evaluate the approaches by measuring 
the mean average precision over target data using SVM.

\subsection{Selecting the optimal dimensionality for the SA method using the
consistency theorem}
\label{sec:subdim2}
In this section, we present the procedure for selecting the space 
dimensionality in the context of the SA method. The same dimensionality is also 
used for 
Baseline-S and Baseline-T. For GFK and GFS we follow the published procedures 
to 
obtain optimal results as presented in~\cite{Gong2012}. First, we perform a PCA 
on the two domains and compute the deviation 
$\lambda_d^{min}-\lambda_{d+1}^{min}$ for all possible $d$  
values. Then, using the theoretical bound of Eq:~\ref{eq:bound}, we can
estimate a $d_{max} << D$ that provides a stable solution with fixed deviation 
$\gamma>0$ for a given confidence $\delta>0$. Afterwards, we consider the
subspaces of dimensionality from $d=1$ to $d_{max}$ and select the best $d^*$ 
that minimizes the classification error using a 2 fold cross-validation
over the labelled source data. This procedure is founded by the theoretical 
result of Ben-David et al. of Eq~\ref{eq:bendavid} where the idea is to try to 
move closer the domain distribution while maintaining a good accuracy on the 
source domain. As an illustration, the best dimensions for the Office+Caltech 
dataset varies between $10-50$. For example, for the DA problem $\mathbf{W}
\to \mathbf{C}$, taking  $\gamma = 10^5$ and $\delta=0.1$, we obtain  
$d_{max} = 22$ (see Figure~\ref{fig:consistency})  and by cross validation we 
found that the optimal dimension is $d^*=20$.

\begin{figure}[t]
 \centering
 \includegraphics[width=55mm]{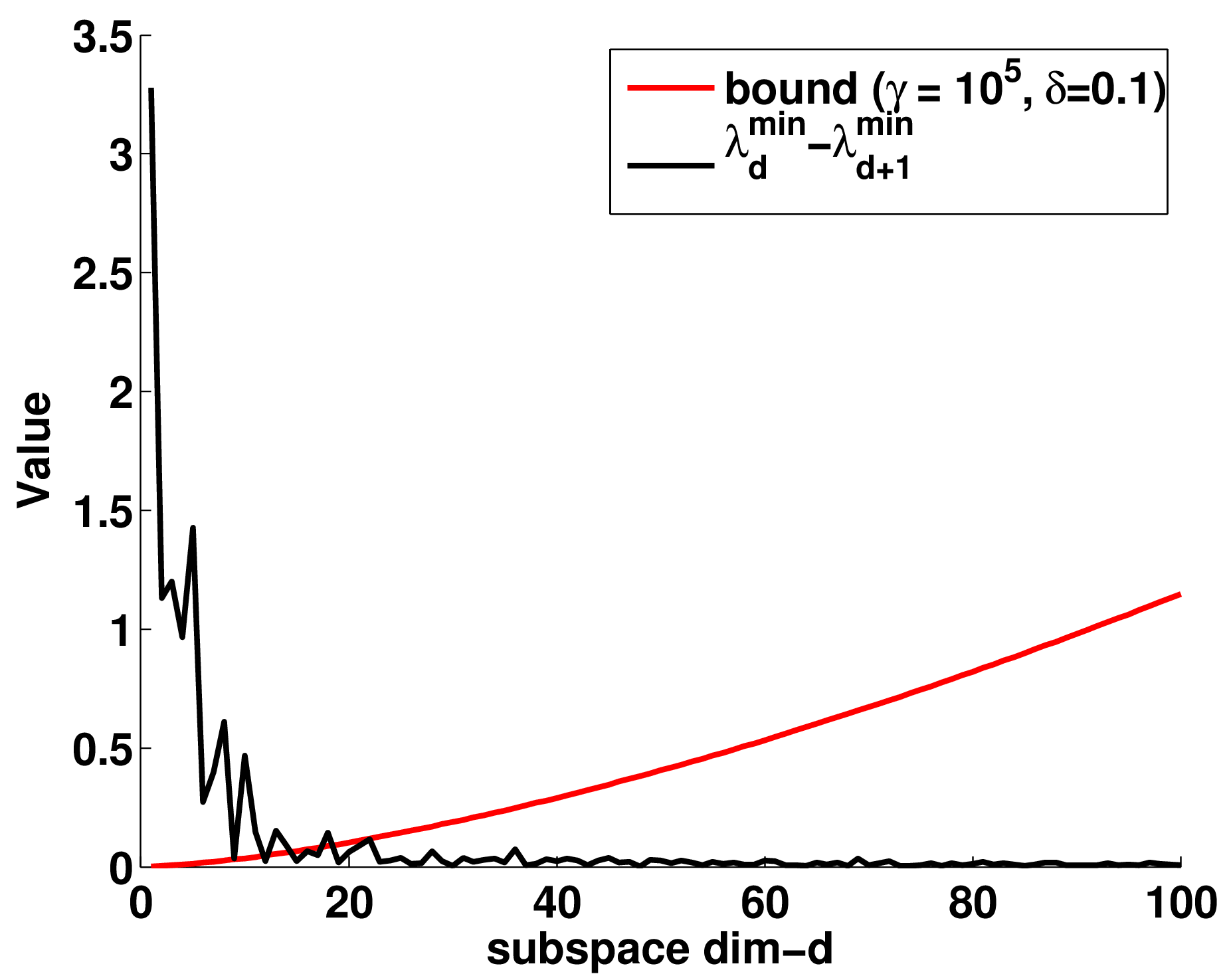}  
 \caption{Finding a stable solution and a subspace dimensionality
   using the consistency theorem. We plot the bounds for  $W \rightarrow C $ DA 
problem taking $\gamma=10^5$ and $\delta=0.1$. The upper bound is plotted in 
red 
color and the difference in consecutive eigenvalues in black color. From this 
plot we select the  $d_{max} = 22$ .  }
 \label{fig:consistency}
\end{figure}

\subsection{Evaluating DA with divergence measures}
\label{sec:divergenceexp}

Here, we evaluate the capability of our SA method to move closer the 
domain distributions according to the measures presented in 
Section~\ref{sec:divergence}: the TDAS adapted to NN classification 
where a high value indicates a better distribution closeness and the $H \Delta 
H $ using a SVM where a value close to 50 indicates close distributions. 

To compute $H \Delta H $ using a SVM we use the following protocol. For each 
baseline method, SA and GFK we apply DA using both source and target data. 
Afterwards, we give label $+1$ for the source samples and label $-1$ for the 
target samples. Then we randomly divide the source samples into two sets of 
equal size. The source train set ($S_{train}$) and source test set 
($S_{test}$). 
We repeat this for the target samples and obtain target train set 
($T_{train}$) and target test set ($T_{test}$). Finally, we train a linear SVM 
classifier using ($S_{train}$) and ($T_{train}$) as the training data and 
evaluate on the test set consisting of ($S_{test}$) and ($T_{test}$). The final 
classification rate obtained by this approach is an empirical estimate of $H 
\Delta H$.  

To compute $TDAS$ we use a similar approach. $TDAS$ is always associated with a 
metric as we need to compute the similarity $Sim(\mathbf{y_S},\mathbf{y_T})=
\mathbf{y_S} A \mathbf{y_T}'$. For Baseline-S, the metric is $XsXs'$. For 
Baseline-T, the metric is $XtXt'$. For GFK we obtain the metric as explained 
in~\cite{Gong2012}. For SA metric is $XsXs'XtXt'$. We set $\epsilon$ to the mean 
similarity between the 
source sample and the nearest target sample.

We compute these discrepancy measures for the 12 DA problems coming from the 
Office and the Caltech-10 datasets and report the mean values over the 12 
problems 
for each method in Table~\ref{tbl:divergence}. We can remark that our approach 
reduces significantly the discrepancy between the source and target domains 
compared to the other baselines (highest TDAS value and lowest $H \Delta H$ 
measure). Both GFK and our method have lower $H \Delta H$ values meaning that 
these methods are more likely to perform well.

\begin{table}[t]
\centering
\small
\begin{tabular}{ | c | c | c | c | c | c |}\hline
Method 		& NA 	& Baseline-S & Baseline-T & GFK  & SA \\ \hline
TDAS   		& 1.25 	& 3.34 	     & 2.74       & 2.84 & \tr{4.26}  \\ \hline 
H$\Delta$H	& 98.1 	& 99.0       & 99.0       & 74.3 & \tr{53.2}  \\ \hline
\end{tabular} 
\caption{Several distribution discrepancy measures averaged over 12 DA 
problems using Office dataset.}
\label{tbl:divergence}
\end{table}

\subsection{Classification Results}
\label{sec:classification-results}

\subsubsection{Visual domain adaptation performance with Office/Caltech10 
datasets}
\label{sec:expOfficecal}

In this experiment we evaluate the different DA methods using 
Office~\cite{Saenko2010}/Caltech10~\cite{Gopalan2011} datasets which consist of 
four domains (\textbf{A}, \textbf{C}, \textbf{D} and \textbf{W}). The results 
for the 12 DA problems in the unsupervised setting using a NN classifier are 
shown in Table~\ref{tbl:office_knn_unsup}. In 9 out of the 12 DA problems our 
method outperforms the other ones. It is interesting to see that projecting to 
the target domain (Baseline-T) works quite well for some DA problems. 
The main 
reason for this could be that projecting to the target subspace allows 
maximizing the mutual information between the projected source and the 
projected 
target data.

The results obtained with a SVM classifier in the unsupervised DA case are 
shown in Table~\ref{tbl:office_svm_unsup}. Our SA method outperforms all the 
other methods in 9 DA problems. These results indicate that our method works 
better than other DA methods not only for NN-like local classifiers  but also 
with more global SVM classifiers.

\begin{table}[t]
\centering
\small
\begin{tabular}{ | c | c | c | c | c | c | c | }\hline
Method & \dt{C}{A} & \dt{D}{A} & \dt{W}{A} & \dt{A}{C} & \dt{D}{C} &  
\dt{W}{C}\\ \hline \hline
NA	   & 21.5    & 26.9    & 20.8    & 22.8    & 24.8    & 16.4 \\\hline
Baseline-S & 38.0    & 29.8    & 35.5    & 30.9    & 29.6    & 31.3 \\ \hline
Baseline-T &\tr{40.5}& 33.0    &\tr{38.0}& 33.3    & 31.2    & 31.9 \\ \hline
GFS 	   & 36.9    & 32      & 27.5    & 35.3    & 29.4    & 
21.7 \\ \hline
GFK  	   & 36.9    & 32.5    & 31.1    &\tr{35.6}& 29.8    & 
27.2 \\ \hline 
TCA & 34.7	& 27.5	& 34.1	& 28.8	& 28.8	& 30.5 \\ \hline

SA   & 39.0    &\tr{38.0}& 37.4    & 35.3    &\tr{32.4}&\tr{32.3}\\\hline 
\hline
Method & \dt{A}{D} & \dt{C}{D} & \dt{W}{D} & \dt{A}{W} & \dt{C}{W} &  
\dt{D}{W}\\ \hline
NA	  & 22.4  & 21.7  & 40.5  & 23.3  & 20.0  & 53.0 \\ \hline
Baseline-S & 34.6  & 37.4  & 71.8  & 35.1  & 33.5  & 74.0 \\ \hline
Baseline-T& 34.7  & 36.4  & 72.9  & 36.8  & 34.4  & 78.4 \\ \hline
GFS       & 30.7  & 32.6  & 54.3  & 31.0  & 30.6  & 66.0 \\ 
\hline
GFK  	  & 35.2  & 35.2  & 70.6  & 34.4  & 33.7  & 74.9 \\ 
\hline 
TCA & 30.4	& 34.7	& 64.4	& 30.3 &	28.8	& 70.9 \\ \hline 

SA  & \tr{37.6}  & \tr{39.6}& \tr{80.3}  & \tr{38.6}  & \tr{36.8}  & 
\tr{83.6} \\ \hline  
\end{tabular} 
\caption{Recognition accuracy with unsupervised DA using NN
classifier (Office dataset + Caltech10).}
\label{tbl:office_knn_unsup}
\end{table}

\begin{table}[t]
\centering
\small
\begin{tabular}{ | c | c | c | c | c | c | c | }\hline
Method & \dt{C}{A} & \dt{D}{A} & \dt{W}{A} & \dt{A}{C} & \dt{D}{C} &  
\dt{W}{C}\\ \hline \hline
NA & 44.0 & 34.6 & 30.7 & 35.7 & 30.6 & 23.4 \\ \hline
Baseline-S   & 44.3 & 36.8 & 32.9 & 36.8 & 29.6 & 24.9 \\\hline
Baseline-T  & 44.5 & 38.6 & 34.2 & 37.3 & 31.6 & 28.4 \\\hline
GFK 	    & 44.8 & 37.9 & 37.1 & 38.3 & 31.4 & 29.1 \\\hline
TCA         & \tr{47.2}  &	38.8&	34.8&	\tr{40.8}&	33.8&	30.9 
\\\hline

SA	& 46.1 & \tr{42.0} & \tr{39.3} & 39.9 & \tr{35.0} & \tr{31.8}  
\\\hline \hline
Method & \dt{A}{D} & \dt{C}{D} & \dt{W}{D} & \dt{A}{W} & \dt{C}{W} &  
\dt{D}{W}\\ \hline
NA  & 34.5 & 36.0 & 67.4 & 26.1 & 29.1 & 70.9 \\ \hline
Baseline-S  & 36.1 & 38.9 & 73.6 & \tr{42.5} & 34.6 & 75.4 \\\hline
Baseline-T  & 32.5 & 35.3 & 73.6 & 37.3 & 34.2 & 80.5 \\\hline
GFK 	      & 37.9 & 36.1 & 74.6 & 39.8 & 34.9 & 79.1 \\ \hline
TCA & 36.4&	39.2&	72.1&	38.1&	36.5&	80.3 \\ \hline

SA   	    & \tr{38.8} & \tr{39.4} & \tr{77.9} & 39.6 & \tr{38.9} & \tr{82.3}  
\\\hline 

\end{tabular} 
\caption{Recognition accuracy with unsupervised DA using SVM
classifier(Office dataset + Caltech10).}
\label{tbl:office_svm_unsup}
\end{table}

\subsubsection{Domain adaptation on ImageNet, LabelMe and Caltech-256 
(\textbf{ILC-5}) datasets} 
\label{sec:expImagenet} 

Results obtained for unsupervised DA using NN classifiers on \textbf{ILC-5} 
datasets are shown in Table~\ref{tbl:imnet_knn_unsup}.  First, it is remarkable 
that
all the other DA methods achieve poor accuracy when LabelMe images are 
used as the source domain (even below NA), while our method seems to adapt the 
source to the 
target reasonably well. On average, our method significantly outperforms all 
other DA methods.

A visual example where we classify ImageNet images using models trained on 
Caltech-256 images is 
shown in Figure~\ref{fig:visual} and Figure~\ref{fig:visual1}. The nearest 
neighbor coming from Caltech-256 corresponds to the same class, even though the 
appearance of images are very different for the two datasets.

In Table~\ref{tbl:imnet_svm_unsup}  we report results using a SVM classifier 
for the unsupervised DA setting. In this case our method systematically 
outperforms all other DA methods, confirming the good behavior of our approach.

\begin{table*}[t]
\centering
\begin{tabular}{ | c | c | c | c | c | c | c |c | }\hline
Method & \dt{L}{C} & \dt{L}{I} & \dt{C}{L} & \dt{C}{I} & \dt{I}{L} & 
\dt{I}{C} & AVG \\ \hline \hline

NA 	 & 46.0 & 38.4 & 29.5 & 31.3 & 36.9 & 45.5 & 37.9 \\ \hline
Baseline-S& 24.2 & 27.2 & 46.9 & 41.8 & 35.7 & 33.8 & 34.9 \\ \hline
Baseline-T& 24.6 & 27.4 & \tb{47.0} & \tr{42.0} & 35.6 & 33.8 & 35.0 \\ \hline
GFK 	 & 24.2 & 26.8 & 44.9 & 40.7 & 35.1 & 33.8 & 34.3 \\ \hline
TCA      & 25.7	& 27.5 & 43.1 &	38.8 & 29.6 & 26.8 & 31.9 \\ \hline
SA  & \tr{49.1} & \tr{41.2} & \tb{47.0} & 39.1 & \tr{39.4} & \tr{54.5} & 
\tr{45.0} \\ \hline
\end{tabular} 
\caption{Recognition accuracy with unsupervised DA with NN
classifier (ImageNet (I), LabelMe (L) and Caltech-256 (C)).}
\label{tbl:imnet_knn_unsup}
\end{table*}

\begin{table*}[t]
\centering

\begin{tabular}{ | c | c | c | c | c | c | c |c | }\hline
Method & \dt{L}{C} & \dt{L}{I} & \dt{C}{L} & \dt{C}{I} & \dt{I}{L} & 
\dt{I}{C} & AVG \\ \hline \hline
NA       & 49.6 & 40.8 & 36.0 & 45.6 & 41.3 & 58.9 & 45.4 \\ \hline
Baseline-S& 50.5 & 42.0 & 39.1 & 48.3 & 44.0 & 59.7 & 47.3 \\ \hline
Baseline-T& 48.7 & 41.9 & 39.2 & 48.4 & 43.6 & 58.0 & 46.6 \\ \hline
GFK      & 52.3 & 43.5 & 39.6 & 49.0 & 45.3 & 61.8 & 48.6 \\ \hline
TCA      & 46.7	& 39.4 & 37.9 &	47.2 & 41.0 & 56.9 & 44.9 \\ \hline
SA  		 & \tr{52.9} & \tr{43.9} & \tr{43.8} & \tr{50.9} & \tr{46.3} & 
\tr{62.8} & \tr{50.1} \\ \hline
\end{tabular} 
\caption{Recognition accuracy with unsupervised DA with SVM
classifier (ImageNet (I), LabelMe (L) and Caltech-256 (C)).}
\label{tbl:imnet_svm_unsup}
\end{table*}

All these results suggest that classifying \textit{Caltech-256} 
images from various other sources seems a difficult task for most of the 
methods (--see Table \ref{tbl:office_knn_unsup} and Table 
\ref{tbl:imnet_knn_unsup}). By analyzing results from section 
\ref{sec:expImagenet}, we can see that ImageNet is a good source to classify 
images from Caltech-256 and LabelMe datasets.

\subsubsection{Classifying PASCAL-VOC-2007 images using classifiers trained
  on ImageNet}
\label{sec:expVOC} 

In this experiment, we compare the average precision obtained on 
PASCAL-VOC-2007 by a SVM classifier in an unsupervised DA setting. We use 
ImageNet as the source domain and PASCAL-VOC-2007 as the target domain. 
The results are shown in Figure~\ref{fig:daimnetvoc}.

Our method achieves the best results for all the categories, GFK improves by 
7\% in mAP over no adaptation while our method improves by 27\% in mAP over GFK.

\begin{figure*}[t]
 \centering
 \includegraphics[width=\textwidth]{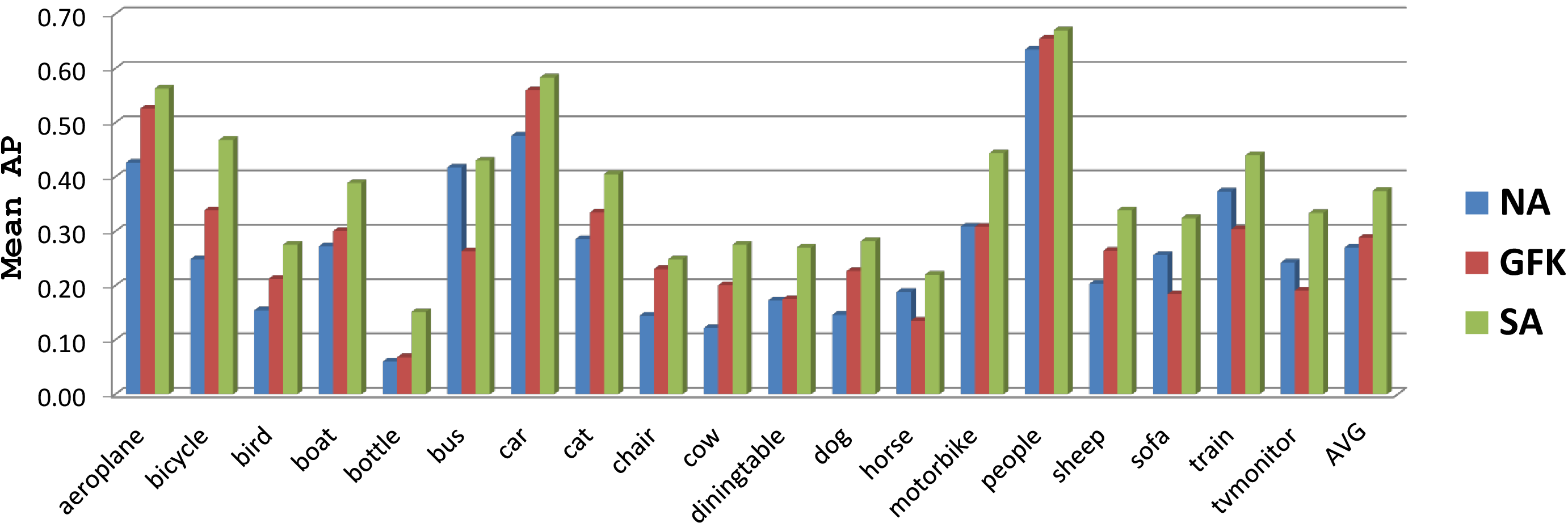}
 \caption{Train on ImageNet and classify PASCAL-VOC-2007 images using 
unsupervised DA with a linear SVM classifier. Average average precision over 20 
object classes is reported.}
 \label{fig:daimnetvoc}
\end{figure*}

In section \ref{sec:expOfficecal},\ref{sec:expImagenet} and \ref{sec:expVOC} we 
evaluate several domain adaptation algorithms using both NN and SVM 
classifiers. 
In all three cases, SA method outperforms TCA, GFS, GFK and baseline 
methods. The target accuracy obtained with the NN classifier is comparable with 
the ones obtained with  the SVM classifier for \textit{Office+Caltech10} 
dataset. On the other hand for $ICL-5$ dataset all methods get a boost in 
results when used with an SVM classifier. As a result in the rest of the 
experiments we 
use NN classifier whenever we use \textit{Office+Caltech10} dataset for 
evaluation.

\subsection{Evaluating methods that use source labels during DA}
\label{sec:expITML}

In this section we evaluate the effect of supervised \textit{ITML-PCA} subspace 
creation method presented in section~\ref{sec:itmlpca} and the large margin 
subspace alignment (LMSA) method introduce in section~\ref{sec:LMSA}. For this 
experiment we use the \textit{Office-Caltech-10} dataset. We compare several 
supervised 
subspace-based DA methods in Table \ref{tbl:itmlpca}, namely PLS, LDA and 
ITML-PCA method introduce in section~\ref{sec:itmlpca}.

\begin{table*}[t]
\centering
\begin{tabular}{  l  c  c  c  c  c  c  c}\hline
Method & \dt{C}{A} & \dt{D}{A} & \dt{W}{A} & \dt{A}{C} & \dt{D}{C} &  
\dt{W}{C}\\ \hline \hline

GFK (PLS,PCA)      & 40.4 & \tr{36.2} & 35.5 & 37.9 & 32.7 & 29.3 &\\ 
GFK (LDA,PCA)      & \tr{41.6} & \tr{36.2} & 39.9	& 31.9 & 31.0 & 
\tr{36.7} &\\
GFK (ITML-PCA,PCA) & 41.0 & 34.9 & \tr{40.3} & \tr{39.2} & \tr{35.4} & 36.0 &\\ 
\hline
SA (PLS,PCA)        & 32.0 & 32.0 & 31.0 & 32.8 & 29.8 &	24.8 &\\
SA (LDA,PCA)&\tr{48.3}&37.6&41.6&35.7&34.3&39.2\\
SA (ITML-PCA,PCA)&47.1& \tr{40.0}&42.4&\tr{41.1}&\tr{38.1}&39.8  \\ 
LMSA  &	45.9 & 36.9 & \tr{43.6} & 40.0 & 35.6 & \tr{40.4}\\ \hline \hline

Method&\dt{A}{D}&\dt{C}{D}&\dt{W}{D}&\dt{A}{W}&\dt{C}{W}&\dt{D}{W} & 
AVG.\\ \hline 
GFK (PLS,PCA)	   &35.1 	& 41.1 &71.2&35.7&35.8&79.1&42.5\\ 
GFK (LDA,PCA)	   &\tr{35.5} 	& 37.1 &68.9&\tr{37.0}&\tr{37.1}&76.9&42.5\\  
GFK (ITML-PCA,PCA) &\tr{35.5}	& 35.1 
&\tr{74.6}&36.1&36.0&\tr{79.8}&\textbf{43.7}\\ \hline
SA (PLS,PCA) 	   &32.3 	& 35.4      &71.1&34.0&34.1&75.0&38.7\\ 
SA (LDA,PCA)       &32.0	&34.0	&58.6&35.2&\tr{46.0}&73.6&43.0\\
SA (ITML-PCA,PCA)  &\tr{43.7}   & 40.4      &\tr{83.0}&\tr{43.5}&42.8& 
\tr{84.5}& \textbf{48.9}\\ 
LMSA & 41.4 & \tr{44.1} & 76.4 & 40.4 & 40.4 & 81.1 & 47.2 \\ \hline
\end{tabular} 
\caption{Recognition accuracy with unsupervised DA using NN classifier (Office 
dataset + Caltech10) using the supervised dimensionality reduction in the 
source 
domain. We compare the effectiveness of ITML-PCA method for both GFK and SA 
methods.}
\label{tbl:itmlpca}
\end{table*}

As can be seen from the Table \ref{tbl:itmlpca}, SA(ITML-PCA,PCA) performs 
better than SA(LDA,PCA) and SA(PLS,PCA). \\
SA(LDA,PCA) reports on average a 
mean accuracy of 43.0\% over 12 DA problems while the SA(ITML-PCA,PCA) 
method reports the best mean accuracy of 48.5 \%. ITML-PCA method also 
improves GFK results by 1.2\% showing the general applicability of this 
approach. GFK(PLS,PCA) and GFK(LDA,PLS) report a mean accuracy of 42.5\% 
while GFK(ITML-PCA,PCA) reports an accuracy of 43.7 \%. This clearly 
shows that the (ITML-PCA) strategy obtains better results for GFK as 
well as for the SA method. SA(PLS,PCA) reports a very poor accuracy of 
38.7\%. This could be due to the fact that when PLS is used, the upper
 bound obtained by consistency theorem does not hold anymore and we 
are likely to obtain a subspace dimension that only works in the 
source domain. In contrast, the ITML-PCA method still uses PCA 
to create a linear subspace and so the consistency theorem holds.

ITML uses both similarity and dissimilarity constraints using an information 
theoretic procedure. The learned metric $W$ in Algorithm (\ref{algo:itmlxs}) is 
regularized by $logdet$ regularization during the ITML procedure. The metric 
learning objective makes the subspace discriminative while allowing the domain 
transfer using SA method. From this result we conclude that $SA(ITML-PCA,PCA)$ 
is a better approach for subspace alignment-based domain adaptation as well as 
GFK method. In future work we plan to incorporate ITML \cite{Davis2007} like 
objective directly in the subspace alignment objective function similar to LMSA 
method.

\subsection{Evaluating SA-MLE method using high dimensional data}
\label{sec:expESA}

\begin{table*}[t]
\centering
\begin{tabular}{ 
|@{\hspace{1mm}}c@{\hspace{1mm}}|@{\hspace{1mm}}c@{\hspace{1mm}}|@{\hspace{1mm}}
c@{\hspace{1mm}}|@{\hspace{1mm}}c@{\hspace{1mm}}|@{\hspace{1mm}}c@{\hspace{1mm}}
|}\hline
Method & D $\rightarrow$ W & W $\rightarrow$ D & A $\rightarrow$ W  & AVG  \\ 
\hline
NA & 62.7 $\pm$ 1.1	& 64.7 $\pm$ 2.2 & 17.1 $\pm$ 2.1 & 48.2 $\pm$ 1.7\\ 
\hline

SA & \tr{71.6 $\pm$ 0.9} & \tr{70.6 $\pm$ 1.6} & 17.0 $\pm$ 1.6 & \tr{53.1 
$\pm$ 1.8} \\ \hline

SA-MLE &68.9$\pm$2.1 &68.0 $\pm$ 1.3  & 16.7 $\pm$ 1.4& 51.2 $\pm$1.6 \\\hline
\end{tabular} 
\caption{
Recognition accuracy\% obtained on the Office dataset when images are encoded 
with SURF features and Fisher Vectors of 64 GMM components.
}
\label{TBL:Office}
\end{table*}

Despite having a large dimensionality, \textit{Fisher vectors} 
\cite{Perronnin2010} have proven to be a robust image representation for image 
classification task. One limitation of subspace alignment is the computational 
complexity associated with high dimensional data (see section \ref{sec:esa}). 
We overcome this issue using \textbf{SA-MLE} method.

In this experiment we also use the \textit{Office} dataset \cite{Saenko2010} 
as well as the \textit{Office+Caltech-10} dataset. The reason is that the 
\textit{Office} dataset consists of more object classes and images (even though 
it has only three domains).  We extract SURF \cite{Bay2008} features to create 
a 
Gaussian mixture model (GMM) of 64 components and encode images using Fisher 
encoding \cite{Perronnin2010}. The GMM is created from Amazon images from the 
\textit{Office} dataset.

Results on the Office dataset using Fisher vectors are shown in Table 
\ref{TBL:Office} while in Table~\ref{TBL:Office+Caltech} results for 
\textit{Office+Caltech-10} is reported. These results indicate the effectiveness 
of 
SA-MLE method considering the fact that it can estimate the subspace 
dimensionality almost 100 times faster than SA method when used with Fisher 
vectors. On average, SA-MLE outperforms no adaptation (NA) results. 
Since MLE method returns two different subspace dimensionality for the source 
and the target, it is not clear how we can apply MLE method on GFK.

Note that SA-MLE is fast and operates in the low-dimen\-sional target subspace. 
More importantly SA-MLE seems to work well with Fisher vectors. The 
computational complexity of SA-MLE  equal to $\mathcal{O}(\hat{d}^2)$ where 
$d$ is the target subspace dimensionality. On the other hand, an estimate of 
the computational complexity of SA is $\mathcal{O}(D^2 \times d_{max})$ due to 
the cross validation procedure necessary to establish the optimal subspace 
dimensionality. Given the fact that SA-MLE is way faster than SA method and 
obtain reasonable results, we recommend SA-MLE for real time DA methods. Also 
we recommend to use SA-MLE method if the dimensionality of data is larger and 
when computational time is an issue.

\begin{table*}[t]
\small
\centering
\begin{tabular}{ 
|@{\hspace{1mm}}c@{\hspace{1mm}}|@{\hspace{1mm}}c@{\hspace{1mm}}|@{\hspace{1mm}}
c@{\hspace{1mm}}|@{\hspace{1mm}}c@{\hspace{1mm}}|@{\hspace{1mm}}c@{\hspace{1mm}}
|@{\hspace{1mm}}c@{\hspace{1mm}}|@{\hspace{1mm}}c@{\hspace{1mm}}|@{\hspace{1mm}}
c@{\hspace{1mm}}|@{\hspace{1mm}}c@{\hspace{1mm}}|@{\hspace{1mm}}c@{\hspace{1mm}}
|@{\hspace{1mm}}c@{\hspace{1mm}}|@{\hspace{1mm}}c@{\hspace{1mm}}|@{\hspace{1mm}}
c@{\hspace{1mm}}|@{\hspace{1mm}}c@{\hspace{1mm}}|}\hline
Method		& C $\rightarrow$ A   	&  D$\rightarrow$A   	&  
W$\rightarrow$A   	&  A$\rightarrow$C   	&  D$\rightarrow$C   &  
W$\rightarrow$C   &  A$\rightarrow$D   &  C$\rightarrow$D   &  W$\rightarrow$D   
&  A$\rightarrow$W   &  C$\rightarrow$W   &  D$\rightarrow$W   &  AVG \\ \hline
NA    &  44.3   &  35.7   &  34.2   &  36.9   &  33.9   &  30.8   &  37.4 &  
42.1   &  76.7   &  35.2   &  32.2   &  82.7   &  43.5\\ \hline
SA & 48.2 & 44.6 & 43.1 & 40.3 & 38.4 & 35.9 & 45.0 & 51.6 & 87.6 & 44.8 & 44.4 
& 90.1 & \textbf{51.2} \\ \hline    
SA-MLE   &  48.0   &  39.0   &  40.0   &  39.9   &  37.5   &  33.0   &  43.9   
&  50.8   &  86.3   &  40.7   &  40.5   &  88.7   &  49.0\\ \hline
\end{tabular} 
\caption{Recognition accuracy\% obtained on the Office+Caltech-10 dataset when 
images are encoded with SURF features and Fisher Vectors of 64 GMM components.}
\label{TBL:Office+Caltech}
\end{table*}

\subsection{Effect of dictionary size on domain adaptation}
\label{sec:anaDictionary}

In this experiment we change the dictionary size of BOVW image representation 
and evaluate the performance on the target domain. For this experiment we use 
the Office-Caltech dataset. The dictionary is created using SURF features 
extracted from \textit{Amazon} images (10 classes) of the \textit{Office} 
dataset using the k-means algorithm. The subspace is created using PCA (no ITML 
is 
applied). For this experiment we use all source images for training and test on 
all 
target images. We compute the average accuracy using NN classifier over all 12 
domain 
adaptation problems for each considered method and plot the mean accuracy in 
Fig. \ref{fig:bowdict}. 

\begin{figure}[t]
 \centering
 \includegraphics[width=\columnwidth]{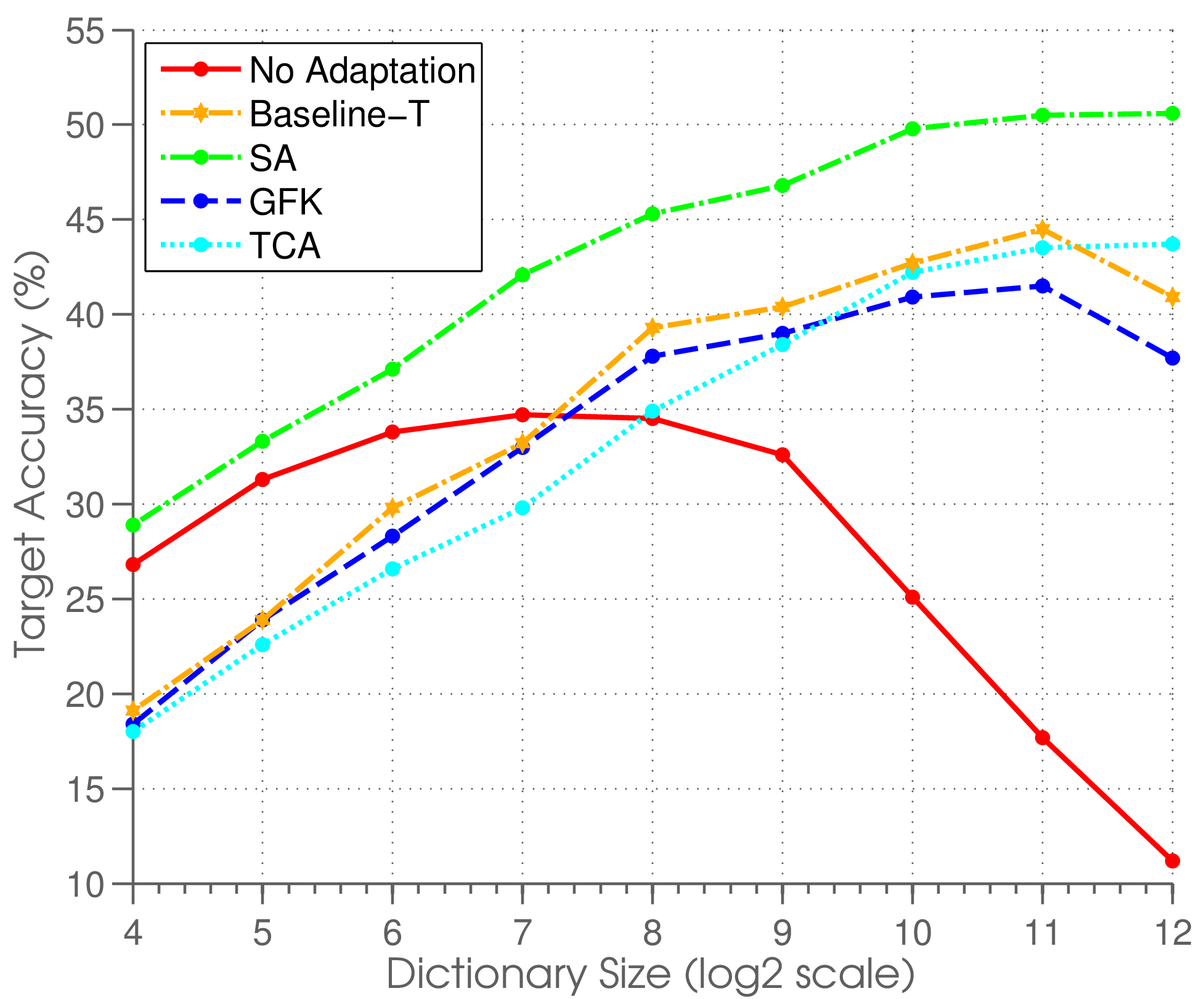}  
 \caption{Mean accuracy on target domain using NN classifier for different 
dictionary sizes on \textit{Office+Caltech-10} dataset (unsupervised domain 
adaptation).}
\label{fig:bowdict}
\end{figure}

From Fig. \ref{fig:bowdict}, we see clearly, the visual dictionary size 
affects the magnitude of domain shift. We can conclude that as the dictionary 
size increases the performance of the target domain increases up-to a point 
and then starts to drop for NA, Baseline-T and GFK. SA method performs the 
best for all dictionary sizes. The typical dictionary size used in the 
literature for this dataset is 800 visual words. But it seems that larger 
visual dictionary of 2048 words or beyond works the best for subspace based DA 
methods 
such as TCA, GFK and SA. This indicates that image representation can influence 
different domain adaptation methods differently.

\subsection{The effect of deep learning features for subspace based DA}
\label{sec:expDecaf}

The divergence between domains depends on the representation used 
\cite{Ben-David2007}. In this experiment we compare the classification 
performance on the target domain using the state of the art image 
classification 
feature called Decaf \cite{DonahueJVHZTD13}. Decaf uses deep learning 
\cite{Krizhevsky2012} based approach. 

We compare the performance of subspace based DA methods using the nearest 
neighbor classifier and SVM classifier. For this experiment we use the more 
challenging Office dataset 
which consists of 31 object classes and three domains (Amazon, Webcam and 
DSLR). 
We show results for NA, GFK, TCA and SA methods using DECAF features in 
Table \ref{tbl:decafnn} for NN classifier and in Table \ref{tbl:decafsvm} for 
SVM classifier.

\begin{table}[t]
\centering
\begin{tabular}{ | c | c | c | c |  }\hline
Method & \dt{D}{W} & \dt{W}{D} & \dt{A}{W} \\ \hline 
NA	& 86.4 $\pm$ 1.1	&88.6 $\pm$ 1.2	& 42.8 $\pm$ 0.9 \\ \hline 
GFK	& 69.2 $\pm$ 2.1	&61.8 $\pm$ 2.3	& 39.5 $\pm$ 2.1\\ \hline 
TCA	& 57.1 $\pm$ 1.7	&51.1 $\pm$ 1.8	& 24.8 $\pm$ 3.2\\ \hline 
SA	&\tr{86.8 $\pm$ 1.0}	&\tr{89.3 $\pm$ 0.8}	& \tr{44.7 $\pm$ 0.7}\\ 
\hline 
\end{tabular} 
\caption{Recognition accuracy with unsupervised DA using NN classifier on 
Office dataset using $Decaf_6$ features.}
\label{tbl:decafnn}
\end{table}

\begin{table}[t]
\centering
\begin{tabular}{ | c | c | c | c |  }\hline
Method & \dt{D}{W} & \dt{W}{D} & \dt{A}{W} \\ \hline 
NA	& 91.3 $\pm$ 1.2	&91.6 $\pm$ 1.6	& \tr{47.9 $\pm$ 2.9} \\ \hline 
GFK	& 87.2 $\pm$ 1.3	&88.1 $\pm$ 1.5	& 46.8 $\pm$ 1.8\\ \hline 
TCA	& 89.0 $\pm$ 1.4	&87.9 $\pm$ 1.9	& 44.6 $\pm$ 3.0\\ \hline 
SA	&\tr{91.8 $\pm$ 0.9}	&\tr{92.4 $\pm$ 1.7}	& 47.2 $\pm$ 1.5\\ 
\hline 
\end{tabular} 
\caption{Recognition accuracy with unsupervised DA using SVM classifier on 
Office dataset using $Decaf_6$ features.}
\label{tbl:decafsvm}
\end{table}

Results in Table \ref{tbl:decafnn} suggest that there is a slight advantage of 
SA 
method over all domains. On the other-hand GFK and TCA seem to perform very 
poorly. When used with SVM classifier both GFK and TCA perform much better than 
with NN classifier (--see Table~\ref{tbl:decafsvm}). Clearly, the DECAF 
features 
boost the performance in typical DA datasets such as Office. However, we notice 
that the performance for $Amazon \rightarrow Webcam$ DA problem is lower. All 
subspace-based DA methods have low-performance for this DA problem when SVM 
classifier is used. However, SA seems to improve over DECAF features for this 
DA task when NN classifier is utilized. We believe there is still room for 
improvements for DA methods to boost recognition rate when DECAF is used as a 
representation.

DECAF being a deep learning feature, it is trained from millions of 
images in a discriminative manner. From learning theory point of view, this 
allows to attenuate the probabilistic upper bound on classification error in 
the 
test set. At the same time, probably as millions of images are used during the 
training, it is possible that training algorithms (Convolutional neural network 
+ upper layers) has already visited variety of images from different domains 
allowing the learning algorithm to be invariant to differences in domains and  
build a representation that is quite domain invariant in general. 

\subsection{Effect of z-normalization}
\label{sec:zscore}

As we mentioned in section \ref{sec:notations}, z-normalization influences DA 
methods such as GFK, TCA and SA. To evaluate this, we perform an experiment 
using \textit{Office+Caltech-10} dataset. In this experiment we use all source 
data for training and all target data for testing. We use the same BOW features 
as in \cite{Gong2012}. We report mean accuracy over 12 domain adaptation 
problems in Fig. \ref{fig:zscore}. We report accuracy with and without 
z-normalization for all considered DA methods.

\begin{figure}[t]
 \centering
 \includegraphics[width=\columnwidth]{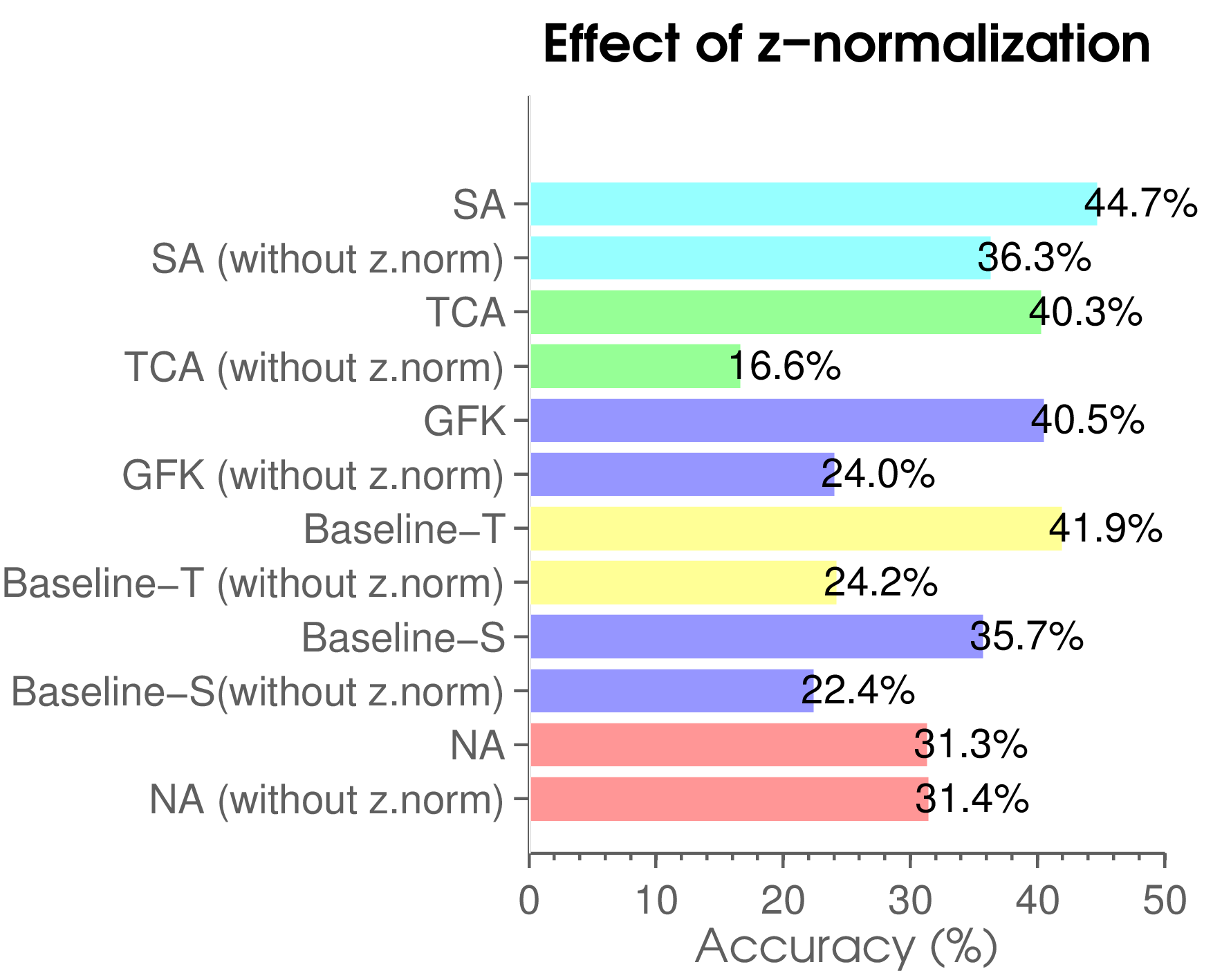}  
 \caption{The effect of z-normalization on subspace-based DA methods. Mean 
classification accuracy on 12 DA problems on Office+Caltech dataset is shown 
with and without z-norm. }
\label{fig:zscore}
\end{figure}

It is clear that all DA method including the methods such as projecting to the 
source (\textit{Baseline-S}) and target domain (\textit{Baseline-T}) hugely 
benefited by the z-normalization. The biggest beneficiary is the TCA method 
while GFK as well as most other methods even fail to improve over no adaptation 
without z-normalization. SA method is benefited from z-normalization but even 
without z-normalization it improves over NA results.

\subsection{Comparison with other non-subspace-based DA methods}
\label{sec:nonsub}

In this section we compare our method with several non-subspace based domain 
adaptation 
methods that exist in the literature. For this evaluation we use the 
Office+Caltech-10 dataset(--see Table~\ref{tbl:asm-state-sup}). 
We compare with traditional self-labeling method similar to 
DA-SVM~\cite{Bruzzone2010}, 
KLIEP-based~\cite{Sugiyama} instance weighting DA method and a dictionary 
learning based 
DA method~\cite{Ni2013}. From results in Table~\ref{tbl:asm-state-sup}, it is 
interesting
 to see that simple Self-labeling methods seems to work well on this 
dataset. Especially, for DA tasks such as ${C} \rightarrow {A}$ and ${C} 
\rightarrow {W}$ self-labeling seems to outperforms SA method. The 
instance weighting method based on KLIEP~\cite{Sugiyama} algorithm did not 
perform well on this dataset.
On the other-hand, methods based on dictionary learning seems to work 
better than SA for some problems such as ${C}\rightarrow{D}$ and 
${D}\rightarrow{W}$.
We believe that a combination of all these methods might lead to superior 
performance. 
In theory, SA method can be combined with instance weighting, self-labeling and 
dictionary learning based methods. 
Such a combination of strategies might assist to overcome practical domain shift 
issues 
in application areas such as video surveillance and vision-based-robotics. 
In future, we plan to investigate the applicability of such combination 
of strategies in real-world DA scenarios.

\begin{table*}[t!]
\small
\centering
\begin{tabular}{ | c | c | c |  c | c | c | c| c| c| c| c| c| c| c|}\hline

Method & \dt{C}{A} & \dt{D}{A} & \dt{W}{A} & \dt{A}{C} & \dt{D}{C} &  
\dt{W}{C} & \dt{A}{D}&\dt{C}{D}&\dt{W}{D}&\dt{A}{W}&\dt{C}{W}&\dt{D}{W} & 
AVG.\\ \hline 

SA (ITML-PCA,PCA) & 47.1& 40.0&42.4&41.1&38.1&39.8&43.7&40.4 &83.0&43.5&42.8& 
84.5&48.5\\ \hline

Self-labelling (SVM) &52.1 & 37.7 & 36.2 & 38.3 & 28.2 & 28.3 & 40.2 & 41.3 & 
70.9 & 39.0 & 46.7 & 81.2 & 45.0 \\ \hline

Instance Wighting (LKIEP+SVM) &43.2 & 34.9 & 30.4 & 36.0 & 31.0 & 23.1 & 24.3 & 
28.0 & 36.7 & 27.3 & 28.2 & 66.4 & 34.1 \\ \hline

Dictionary Learning \cite{Ni2013}& 45.4 & 39.1 & 38.3 & 40.4 & N/A & 36.3 & N/A 
& 42.3 & N/A & 37.9 & N/A & 86.2 & 45.7 \\ \hline


\end{tabular} 
\caption{Comparison with non subspace-based methods with SA(ITML-PCA, PCA) 
method using Office+Caltech-10 dataset using bag-of-words features.}
\label{tbl:asm-state-sup}
\end{table*}

\section{Discussion and Conclusion}
\label{sec:conclusion}

We have presented a new visual domain adaptation method using subspace 
alignment. In this method, we create subspaces for both source and target 
domains and learn a linear mapping that aligns the source subspace with the 
target subspace. This allows us to compare the source domain data directly with 
the target domain data and to build classifiers on source data and apply them 
on 
the target domain. We demonstrate excellent performance on several image 
classification datasets such as Office dataset, Caltech, ImageNet, LabelMe and 
Pascal-VOC2007. We show that our method outperforms state-of-the-art domain 
adaptation methods using both SVM and nearest neighbour classifiers. Due to its 
simplicity and theoretically founded stability, our method has the potential to 
be applied on large datasets. For example, we use SA to label PASCAL-VOC2007 
images using the classifiers build on ImageNet dataset.

In this extended version of our original work, we addressed several limitations 
of our previous paper. First, we propose SA-MLE method which does not require 
any cross-validation to find the optimal subspace dimensionality. SA-MLE uses 
maximum likelihood estimation to find a good source and target subspace 
dimensionality. Secondly, we propose a new way to use label information of the 
source domain to obtain more discriminative source subspace. Thirdly, 
we provided some analysis on domain adaptation from representation point of 
view. We showed that the dictionary size of the \textit{bag-of-words} 
representation influences the classification performance of subspace-based DA 
methods. Generally, larger visual vocabulary seems to improve domain adapted 
results even when results without adaptation is significantly low. At the same 
time we see that new representations such as DECAF can be used to improve 
recognition rates over different domains. We conclude that visual domain 
adaptation can be overcome by either choosing generic image representations or 
by developing better DA algorithms such as subspace alignment.

Our method assumes that both the source and the target data lie in the 
same space and joint distributions to be correlated. SA method exploits 
this correlation, if there is any to do the adaptation. If both 
joint-distributions 
are independent, then it would be impossible for our method to learn any 
meaningful adaptation.

The subspace alignment method assumes both source and target domain have 
similar class prior probability distributions. When there is significant change 
in class prior probability (as recently discussed in target shifts 
\cite{Zhang2013}), the subspace alignment method may fail. Also in future work 
we plan to incorporate 
information theoretic metric learning like domain constraints directly 
in the subspace alignment objective function.

\textbf{Acknowledgements:}\\
The authors acknowledge the support of the FP7 ERC Starting Grant 240530 
COGNIMUND, ANR LAMPADA 09-EMER-007-02 project and PASCAL 2 network of 
Excellence.\\

%

\end{document}